\documentclass{arxiv}

\PassOptionsToPackage{numbers, compress}{natbib}

\title{Bayesian Hierarchical Invariant Prediction}

\usepackage{mathtools} 

\usepackage[utf8]{inputenc} 
\usepackage[T1]{fontenc}    
\usepackage{hyperref}       
\usepackage{url}            
\usepackage{booktabs}       
\usepackage{amsfonts}       
\usepackage{times}

\usepackage{nicefrac}       
\usepackage{microtype}      
\usepackage{xcolor}         
\usepackage{cleveref}

\newtheorem{assump}{Assumption}
\newtheorem*{proof*}{Proof}
\newtheorem{prop}{Proposition}
\newtheorem*{prop*}{Proposition}
\newtheorem{thm}{Theorem}
\newtheorem*{thm*}{Theorem}

\usepackage{multirow}
\usepackage{graphicx} 
\usepackage{subcaption} 
\usepackage{wrapfig}

\usepackage{tikz} 
\usetikzlibrary{bayesnet}
\newcommand{\E}{\mathcal{E}}

\newcommand{\pa}{\normalfont \text{PA}}

\newcommand{\p}{P}

\usepackage[toc,page,header]{appendix}
\usepackage{minitoc}

\usepackage{bm}

\clearauthor{%
 \Name{Francisco Madaleno}
 \Email{\href{mailto:fmfsa@dtu.dk}{fmfsa@dtu.dk}} \\
 \Name{Pernille Julie Viuff Sand}\\
 \Name{Francisco C. Pereira}\\
 \addr Department of Technology, Management and Economics, Danish Technical University
 \AND
 \Name{Sergio Hernan {Garrido Mejia}} \\
 \addr Max Planck Institute for Intelligent Systems %
}

\begin{document}

\doparttoc 
\faketableofcontents 

\maketitle

\begin{abstract}%
We propose Bayesian Hierarchical Invariant Prediction (BHIP) reframing Invariant Causal Prediction (ICP) through the lens of Hierarchical Bayes. We leverage the hierarchical structure to explicitly test invariance of causal mechanisms under heterogeneous data, resulting in improved computational scalability for a larger number of predictors compared to ICP. Moreover, given its Bayesian nature BHIP enables the use of prior information. We evaluate BHIP on both synthetic and real-world datasets, demonstrating its potential as an alternative inference method to ICP and related methods.
\end{abstract}


\section{Introduction}\label{sec:intro}

Heterogeneous data is ubiquitous in real-world applications, spanning fields such as medicine, transportation, and environmental science. 
The critical observation of this paper is that the heterogeneous data forming the basis of Bayesian Hierarchical Models (BHM) corresponds to “environment data” described in Invariant Causal Prediction (ICP) \citep{peters2016invariant}. 
Both perspectives acknowledge that real-world data often arises from distinct underlying processes or ‘environments,’ underscoring the need for robust models that account for these differences and ultimately improve out-of-distribution performance. 
For example, data collected across hospitals, regions, or time periods often exhibit distribution shifts: patient populations and treatment protocols may differ across clinics, and travel behavior may differ across cities or seasons.

Without specific structural assumptions observational data can only allow inferring a Directed Acyclic Graph (DAG) model up to the Markov Equivalence Class \citep{pearl2000causality}.
ICP identifies causal predictors by leveraging invariance properties across environments, exhaustively employing conditional independence tests to find sets of predictors for which the target variable is independent of the environment when conditioned on those predictors. 
As most statistical methods, ICP is sample-size sensitive, but also computationally complex since it involves testing for invariance across all possible subsets of predictors, $2^d$ where $d$ is the number of predictors.
Additionally, ICP is designed to be conservative and minimize Type I errors which often leads to low power \citep{peters2016invariant}.

Addressing these limitations, we introduce Bayesian Hierarchical Invariant Prediction (BHIP), which models the variability of predictors' effects across environments, quantifies uncertainty, and allows for the incorporation of priors (which we explore in Appendix \ref{subsec:sparse} with the inclusion of sparsity inducing priors), enhancing flexibility and overall performance compared to frequentist approaches like ICP. 
While BHIP utilizes established Bayesian techniques, namely hierarchical modeling, its novelty lies not in the individual components themselves, but in their specific synthesis and application to probabilistically test causal invariance for parent discovery. 
Specifically, this work contributes a unified probabilistic framework for invariant prediction. 
Leveraging its hierarchical structure and Bayesian nature, it simultaneously accounts for environmental heterogeneity, allows variable selection (e.g., through sparsity-inducing priors), quantifies uncertainty in estimated effects and, crucially, enables probabilistic invariance testing directly within the model's posterior analysis.

By repurposing the standard hierarchical modeling structure and combining it with specific posterior decision rules, BHIP directly identifies invariant parents in heterogeneous data. 
This targeted and integrated Bayesian approach represents a significant step beyond simply applying standard Bayesian methods or relying solely on frequentist invariance testing. 
Code is available at 
\href{https://github.com/fmfsa/bhip}{GitHub repository}
\footnote{\url{https://github.com/fmfsa/bhip}}.
 
\subsection{Related Work}\label{sec:related_word}

The problem of \textbf{causal feature selection} or causal discovery on a specific target variable has been studied as an alternative to the more difficult problem of full causal discovery.
Under Pearl's view of Causality \citep{pearl2000causality}, ICP laid down the foundations for causal feature selection using the idea of invariance.
What differentiated ICP from other methods for causal discovery using both interventional and observational data is that ICP does not require the analyst to know where the interventions are performed, but instead only to know from which setting a particular data point comes.
This powerful idea was later extended to non-linear models \citep{heinze2018invariant}, dynamical systems \citep{pfister2019learning}, time series data \citep{pfister2019invariant}, spatio-temporal data \citep{christiansen2022toward} and different outcome models \citep{kook2024model}. Moreover, the invariant principle has been also used together with adjacent machine learning methods: active learning \citep{gamella2020active}, to find the causal features with experimental data efficiently; and reinforcement learning \citep{saengkyongam2023invariant}, for policy learning.

\textbf{Bayesian inference} has also been used for causal discovery, \citep{heckerman2006bayesian} being one of the earliest examples with Monte Carlo inference on causal graphs, all the way to recent advances, such as \citep{hagele2023bacadi} who use differentiable methods to infer causal structure in a Bayesian way.
The theoretical foundations of the Bayesian method have also helped gain insights in causality using the invariance principle in exchangeable data \citep{guo2024causal,guo2024finetti}, and understanding causality in the context of hierarchical models \citep{blei2024hierarchicalcausalmodels}.

The relation between hierarchical Bayes and causal discovery is long-standing.
\citep{gelman2013ask} talk about the differences between asking about effects of causes (causal inference) and causes of effects (causal discovery) and mention how bayesian, and specifically hierarchical, models can be used for the latter (they present \citep{manton1989empirical} as an example).

Concurrently, \cite{wu2025bayesianinvariancemodelingmultienvironment} develop a Bayesian invariant prediction model that, alike this work, casts ICP-style invariance as posterior inference over invariant features, but differs by encoding the invariant set through a latent feature-selection variable in a joint generative model with consistency and variational-inference guarantees, rather than via hierarchical partial pooling on regression coefficients as in BHIP.

\section{Background}\label{sec:background}

\subsection{Graphical Models and Causal Inference}\label{subsec:gm}

In this paper we use SCMs \citep{elementsCausalInference} to represent causal relations.

\begin{definition}[Structural Causal Model (SCM)]
A $D-$dimensional SCM is a tuple $\mathcal{M}:=(\mathbb{S}, P_\varepsilon)$ consisting of a set $\mathbb{S}$ of structural assignments
\begin{align}
    X_d := f_d(\pa(X_d),\varepsilon_d)), \text{ for } d=1,\dots,D,
\end{align}
where $\pa(X_d)\subseteq\{X_1,\dots,X_D\}\setminus\{X_d\}$ are called the parents or causes of $X_d$,
and a joint distribution $P_\varepsilon = P_{\varepsilon_1,\dots,\varepsilon_D}$ over noise variables that we assume to be independent.
\end{definition}

We can obtain a causal graph $\mathcal{G}$ from the SCM by drawing a vertex for each $X_d$ and a directed edge from each vertex $X_i\in\pa(X_d)$ to $X_d$.
We assume the obtained causal graph is a DAG.
Likewise, even though the causal mechanisms are deterministic functions of its inputs, we obtain a distribution of each of the variables of our system, $P_{X_d}$, by considering the pushforward distribution of $\pa(X_d)$ and $\varepsilon_d$.
We exploit the independent \citep{janzing2010causal} and invariant \citep{aldrich1989autonomy, peters2016invariant} nature of causal mechanisms to identify the causes of a target variable, $Y$, using heterogeneous data.

\subsection{Invariant Causal Prediction}\label{subsec:icp}

Traditionally, Machine Learning methods assume data is \textit{independent and identically distributed (i.i.d)}, unless specified by a particular structure like time series or graph models.
Researchers in causality have found ways of exploiting heterogeneous data to do causal discovery \citep{cooper1999causal, mooij2020joint, brouillard2020differentiable, peters2016invariant}. 
In particular, ICP aims to estimate a target's set of causal parents by leveraging the property that the conditional of $Y$ given its direct causes $\pa(Y)$ remains invariant across heterogeneous data collection, assuming the causal mechanism of $Y$ is not affected \citep{peters2016invariant}. 

Throughout this work, the discrete heterogeneous contexts are called \textit{environments} $e\in\mathcal{E}$, where $\mathcal{E}$ is an index set, each consist of i.i.d data of a \textit{target} variable of interest $Y$ and $D$ covariates (or predictors) $\mathbf{X}=(X_1,\dots,X_D)$, that is,
\begin{align}
    (Y^e, \mathbf{X}^e) \text{ for } e\in\mathcal{E}. \label{eq: hetero data}
\end{align}

ICP is based on a fundamental invariance assumption, along with the notion of an \textit{invariant set of predictors}, formulated as follows.
\begin{assump}\label{assump:inv}(Invariance)
Given environments $\mathcal{E}$, for a subset of indices $S \subseteq \{1, \dots, D\}$, there exists a subset $X_S$ of covariates such that,
\begin{equation}
    Y^e \mid (X^e_S = x) \overset{d}{=} Y^f \mid (X^f_S = x),
\end{equation}
 for all $e, f \in \mathcal{E}$ and all $x$, that is, the conditionals of $Y$ given $X_S$ are equal in distribution in all environments.
\end{assump}

Any subset $S \subseteq \{1, \dots, D\}$ for which Assumption \ref{assump:inv} holds, is called \textit{invariant with respect to $\mathcal{E}$} and the set of covariates $X_S$ are denoted \textbf{\textit{invariant predictors}}. 
The invariance assumption can be equipped with more structure, assuming a system induced by an SCM:

\begin{assump}(Structural invariance)\label{assump:strinv}
For an SCM and environments, $\mathcal{E}$, the structural equation for $Y$ remains the same across $\mathcal{E}$, and the distribution of $X$ is allowed to change.
That is, for all $e \in \mathcal{E}$,
\begin{align}
    X^e &\sim P_X^e,\\
    Y^e &= f_Y(X^e_{\pa(Y)}, \varepsilon_Y), \quad \varepsilon \sim P_{\varepsilon_Y}, \quad \varepsilon_Y \perp\!\!\!\perp X^e_{\pa(Y)},
\end{align}
$X^e$ represents the covariate $X$ in the environment $e\in \mathcal{E}$, $P_{\varepsilon}$ remains the same and $P_X^e$ can differ.
\end{assump}

Under this structure, the set of direct causes $\pa(Y)$ satisfy invariance \citep{buhlmann2018}, further explained in Appendix \ref{appendix:proposition}. 
ICP tests the invariance of all possible covariate subsets $S$. 
The set of \textit{identifiable causal predictors} $\hat{S}$ is the intersection of the subsets that pass as invariant. 
The main theorem in \citep{peters2016invariant} states that with a controllable coverage probability, the method recovers the set of true causal parents, that is, with desired $\alpha\in(0,1)$, we have $ \mathbb{P}[ \hat{S} \subseteq \pa(Y)]\geq 1-\alpha.$

The intersection controls against Type I errors and often leads to a conservative estimate, given that in the case where the amount of heterogeneity is insufficient, the intersection is $\hat{S}=\emptyset$. In \Cref{subsec:testing} we introduce a Bayesian test for invariance based on the BHIP that serves as an alternative to the test proposed in \citep{peters2016invariant}.

\subsection{Bayesian Hierarchical Models}\label{subsec:bhm}

Hierarchical models estimate posterior distributions, accounting for structured heterogeneity \citep{GelmanHill2007}. 
With heterogeneous data as in \Cref{eq: hetero data}, a hierarchical model contains both parameters related to variation within each environment, termed local-level parameters or environmental-specific  $\beta^e$ for all $e\in\E$, and parameters related to variation between environments, termed global parameters $\phi$. 

In BHMs, we impose a distribution on the latent global parameter $\phi$; the joint probability over the data and model parameters is then given by the following factorization:

\begin{align}
    \p(Y^e, \mathbf{X}^e,\mathbf{\bm\beta}^e,\phi) 
    = \p(\phi)
    \prod_{e\in\E} 
    \p(\beta^e \mid \phi)
    \p(Y^e\mid \mathbf{X}^e,\beta^e).
\end{align}


Bayesian models integrate domain knowledge through priors. 
With Bayesian inference of the global parameter distribution jointly with local level parameters it is possible to analyze the strength of pooling consistent with the observed data \citep{Betan}. 
We reinterpret this strength of pooling forming an invariant perspective to estimate ICP with a BHM: if, for a specific covariate, the global and local parameters are different from zero and `close' (see \Cref{subsec:testing} for a precise definition) to each other, then we can think of them as invariant.

\section{Bayesian Hierarchical Invariant Prediction}\label{sec:methodology}
Consider a dataset comprising of data from a set of environments $\mathcal{E}$ as in 
\Cref{eq: hetero data} induced by an SCM where Assumption \ref{assump:strinv} holds. 
The heterogeneity is modeled in a BHM allowing global parameters to be random while simultaneously modeling individual-level effects. 
The goal of BHIP is to identify covariates where the individual effects on $Y$ are non-zero and remain the same across environments; 
that is, the relevant parameters that remain invariant. 
In addition to Assumption \ref{assump:strinv}, BHIP relies on a set of assumptions common in ICP and Bayesian modeling. A detailed list can be found in Appendix \ref{appendix:assumptions}.

\subsection{Probabilistic Model}\label{subsec:model}
First, we specify the hierarchical model where we can draw inferences about the influence of each covariate on the target and how it varies between environments. Assume the data generation process of $Y^e$ for a given environment $e \in \mathcal{E}$ is governed by environment-specific parameters $\boldsymbol{\beta}^e = (\beta_1^e, \ldots, \beta_D^e)$, explaining the effects of covariates $\boldsymbol{X}^e = (X_1^e, \ldots, X_D^e)$ on the target $Y^e$,
\begin{align}
    Y^e \mid \boldsymbol{X}^e, \boldsymbol{\beta}^e \sim \p(Y^e \mid \boldsymbol{X}^e, \boldsymbol{\beta}^e), \text{ for } e \in \mathcal{E}.
\end{align}

For each covariate index $d \in \{1, \ldots, D\}$, it is assumed that a common global-level distribution with parameters $\phi_d$ generates a local-level parameter $\beta_d^e$ with regards to each environment,
\begin{align}
    \beta_d^e \mid \phi_d \sim \p(\beta_d^e \mid \phi_d), \text{ for } e \in \mathcal{E}.    
\end{align}

To obtain a hierarchical model we treat the global parameters as unknown and impose a (hyper~)prior distribution on the global parameters:
\begin{align}
    \phi_d \sim \p(\phi_d), \text{ for } d=1,\dots,D.
\end{align}

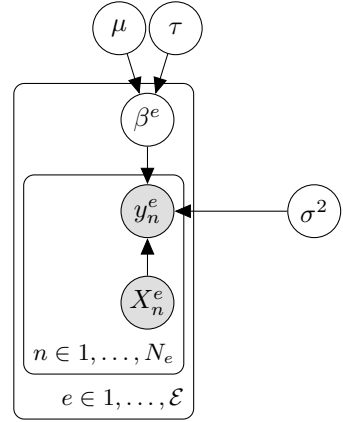
\begin{wrapfigure}{r}{0.4\textwidth}
    \vspace{-\intextsep}
    \centering
    \begin{tikzpicture}[x=2cm,y=0.5cm]

  \node[obs, xshift=-0.5cm]        (X) at (0, 0) {$\boldsymbol{X}_n^e$};
  \node[obs, above=of X] (Y) {$y_n^e$};
  
  \node[latent, above=of Y] (BetaE) {$\boldsymbol{\beta}^e$};

  \node[latent, above=of BetaE, xshift=-0.375cm] (Mu) {$\boldsymbol{\mu}$};
  \node[latent, above=of BetaE, xshift=0.375cm] (Tau) {$\boldsymbol{\tau}$};
  \node[latent, right=of Y, xshift=-0.5cm] (Sigma) {$\sigma^2$};

  \edge {X} {Y};
  \edge {BetaE} {Y};
  \edge {Mu} {BetaE};
  \edge {Tau} {BetaE};
  \edge {Sigma} {Y};

  \plate {env} {(X)(Y)} {$n \in 1,\dots,N_e$};

  \plate {pop} {(env)(BetaE)} {$e \in 1,\dots,\mathcal{E}$};

\end{tikzpicture}
\caption{Bayesian Hierarchical PGM}
\label{fig:pgm}
    \vspace{-\intextsep}
    \vspace{-\intextsep}

\end{wrapfigure}

As an example, consider the following \textbf{normal linear hierarchical model}.
We assume a Gaussian likelihood for $y_n^e$. 
For the local parameters $\boldsymbol{\beta}^e$, we assume each predictor's coefficient $\beta_d^e$ is drawn from a Gaussian distribution governed by predictor-specific global parameters: 
a mean $\mu_d$ and a scale $\tau_d$. These form the global vectors $\boldsymbol{\mu} = (\mu_1, \ldots, \mu_D)$ and $\boldsymbol{\tau} = (\tau_1, \ldots, \tau_D)$
\footnote{If $y$ is not continuous, the likelihood in step 2.(a)i. should be adjusted, for example, for categorical $y$, use $y_n^e \sim \text{Categorical}(\text{softmax}(\boldsymbol{X}_n^e \boldsymbol{\beta}^e)).$}.
The generative process of BHIP model represented in Figure \ref{fig:pgm} as a Probabilistic Graphical Model (PGM) can then be summarized as follows:
\begin{enumerate}
  \item Draw global parameters:
  \begin{itemize}
    \item $\boldsymbol{\mu} \sim \mathcal{N}(\mu_0,\Sigma_0)$
    \item $\boldsymbol{\tau} \sim \text{Half-Cauchy}(\sigma_0)$
  \end{itemize}

  \item For each environment $e \in \{1,\dots,E\}$:
  \begin{enumerate} 
    \item Draw local parameters:
    \begin{itemize}
      \item $\boldsymbol{\beta^e} \sim \mathcal{N}(\boldsymbol{\mu},\boldsymbol{\tau})$
    \end{itemize}

    \begin{enumerate} 
      \item For each observation $n \in \{1,\dots,N_e\}$:
      \begin{itemize}
        \item Draw $y_n^e \sim \mathcal{N}(\boldsymbol{X}_n^e \boldsymbol{\beta}^e, \sigma^2)$
      \end{itemize}
    \end{enumerate}
  \end{enumerate}
\end{enumerate}

\subsection{Inference and Testing Procedures}\label{subsec:testing}

This section describes how we identify invariant predictors within the BHIP framework using our BHM. 
Our definition of invariance is based upon on a predictor's global and local parameters being both credibly non-zero and `close' to each other.

To assess this relationship formally, we employ specific Bayesian statistical tests. 
This approach provides a more integrated and potentially more efficient way to assess invariance directly from estimated parameter relationships, offering a distinct advantage over ICP that requires exhaustive conditional independence tests. 
We fit the BHM to obtain posterior distributions for the parameters (See \ref{appendix:non_centered} for details), and then apply our proposed tests to the posterior samples for each predictor to classify them as invariant based on our predefined criteria. 
We employ two key statistical tests:

\textbf{HDI+ROPE Decision Rule} is used for the `effect non-zero test,' which determines whether a predictor consistently influences the target across environments. 
It combines the Highest Density Interval (HDI) with the Region of Practical Equivalence (ROPE) \citep{hdirope}. 
ROPE is an interval around zero representing negligible effects.
A predictor is considered relevant if its HDI lies substantially outside this region. 
For a parameter $\theta$, the procedure follows these steps:

\begin{enumerate}
    \item Define the ROPE as $[-\epsilon, \epsilon]$, where $\epsilon = 0.1 \cdot \hat{s}$ and $\hat{s}$ is the sample standard deviation of the posterior distribution of $\theta$.
    \item Compute the HDI. Calculate a standard HDI (e.g., a $95\%$ HDI) from the posterior distribution. This gives a specific interval $[\theta_\text{lower}, \theta_\text{upper}]$, which represents the most credible values for the parameter.
    \item Compare the HDI to the ROPE. The position of the computed HDI relative to the ROPE determines the outcome of the test based on the following decision rule:
    \begin{itemize}
        \item Rejected: The effect is considered statistically significant if the entire HDI is outside the ROPE.
        \item Accepted: The effect is considered practically equivalent to zero if the entire HDI is inside the ROPE.
        \item  Undecided: The test is inconclusive if the HDI and the ROPE partially overlap.
    \end{itemize}

\end{enumerate}

The result from this procedure is a categorical decision (`Rejected', `Accepted', or `Undecided') about the predictor's effect. We also compute the largest probability mass outside of ROPE, which is a value in $[0,1]$. Where higher values represent higher confidence in rejecting the hypothesis that the predictor does not have an effect on the target variable. 

\textbf{Pooling Factor} quantifies information sharing (pooling) across environments in a BHM, as proposed by \citep{gelman2006pardoe}.
Define for each environment and each covariate $X_d$ the error term $\delta_d^e$, which is the difference between the global mean and the local parameter, that is $\beta_d^e = \mu_d + \delta_d^e$. It is defined as,

\begin{align}
    \gamma_d = 1 - \frac{\mathrm{Var}_{e\in\E} \left[\mathbb{E} \left[\delta_d^e\right]\right]}{\mathbb{E}\left[\mathrm{Var}_{e\in\E} \left[\delta_d^e\right] \right]}.
\end{align}


The denominator, $\mathbb{E} \left[ \mathrm{Var}_{e \in \mathcal{E}} (\delta_d^e) \right]$, is the expected variance in the deviations from the environment-level effect and the global parameter. This is the unexplained component of the variability in the $\beta^e$'s. Interpreting the pooling factor $\gamma_d$: values close to $1$ signal strong invariance, suggesting the variable is a potential invariant predictor if its effect is non-zero. The lower the value of $\gamma_d$ (i.e., the further it is below $1$), the higher the indicated heterogeneity, suggesting the variable is less likely to be invariant.

In an applied setting, we rely on both the HDI+ROPE decision rule and the pooling factor test together. 
Requiring a predictor to pass both tests embodies the core idea that invariant parameters must have credibly non-zero effects across different environments.

Under the invariance assumption and regularity conditions we can guarantee that as we collect more data and environments, the pooling factor converges to 1 for those variables that are causal parents.
The full proof of our main result can be found in Appendix \ref{proof:consistency}).

 \begin{thm}(Asymptotic behavior of pooling factor)
 Suppose the invariance assumption (Assumption \ref{assump:inv}) and the Bernstein von Mises holds for a BHIP model, then
\begin{align}    
\gamma_d 
\xrightarrow{P}
\begin{cases}
1 & \text{if } \beta^{e*}_d = \beta^{e'*}_d \text{ for all } e,e'\in \mathcal{E}  \\
0 & \text{otherwise.} 
\end{cases}
\end{align}
\label{proposition:guarantee}
\end{thm}

\section{Experiments and Results}\label{sec:results}

This section presents several experiments to evaluate the performance and characteristics of BHIP. In \Cref{subsec:bus_dwelling} we introduce the Bus Dwelling Problem, a controlled setting where we define the SCM and its corresponding DAG, allowing us to assess BHIP’s ability to recover causal relationships and compare its performance against ICP. \Cref{subsec:educational} focuses on the educational attainment dataset \citep{rouse1995} that was used to introduce ICP \citep{peters2016invariant}. In \Cref{subsec:computational}, we provide an analysis of the computational scalability of BHIP compared to ICP, demonstrating a key advantage for problems with many predictors. After establishing these properties, in \Cref{subsec:bip_synth} and \Cref{subsec:bip_real} we benchmark BHIP against other invariant prediction methods, including the concurrent work of \citet{wu2025bayesianinvariancemodelingmultienvironment}, on the low-dimensional synthetic and gene perturbation setups introduced in that work. Further experiments were conducted and can be found in Appendices \ref{synthetic_results} and \ref{tu_results}. The former focuses on a synthetic problem with many different configurations for a quantitative analysis of BHIP's performance, while the latter focuses on a transport-related problem with real data, where BHIP is applied to infer causal predictors in the choice of mode of transport. All experiments in this work were run on a Threadripper 2950X (40M Cache, 3.4 GHz base) CPU, 16 cores, 128 GB RAM.

\subsection{Case Study: The Bus Dwelling Problem}\label{subsec:bus_dwelling}

The Bus Dwelling Problem is a controlled experiment where we model the time a bus dwells at different bus stops as functions of multiple factors: time of the day $X_0$, day of the week $X_1$, traffic conditions $X_2$, and the number of boarding $X_3$ and alighting passengers $X_4$. The DAG that represents our SCM is represented by Figure \ref{fig:dag_bus}. Unlike real-world datasets, this setup allows us to define the ground-truth causal relationships and systematically evaluate how well BHIP recovers them.

\begin{figure}[htbp] 
    \centering 

    \subfigure[Synthetic bus dwelling problem causal graph.\label{fig:dag_bus}]{
    \resizebox{0.48\textwidth}{!}{%
      \begin{tikzpicture}[x=0.85cm, y=0.65cm, node distance=0.7cm, auto, transform shape]
        \node[draw, circle, font=\bfseries] (X0) at (-6, 1.5) {$X_0$};
        \node[draw, circle, font=\bfseries] (X1) at (-6, -1.5) {$X_1$};
        \node[draw, circle, font=\bfseries] (X2) at (-3, 2.33) {$X_2$};
        \node[draw, circle, font=\bfseries] (X3) at (-3, 0) {$X_3$};
        \node[draw, circle, font=\bfseries] (X4) at (-3, -2.33) {$X_4$};
        \node[draw, circle, font=\bfseries] (Y)  at (0, 0) {$Y$};
        \draw[->, thick] (X0) -- (X2);
        \draw[->, thick] (X1) -- (X2);
        \draw[->, thick] (X0) -- (X3);
        \draw[->, thick] (X1) -- (X3);
        \draw[->, thick] (X0) -- (X4);
        \draw[->, thick] (X1) -- (X4);
        \draw[->, thick] (X3) -- (Y);
        \draw[->, thick] (X4) -- (Y);
      \end{tikzpicture}%
    }
    }
  \hfill
  \subfigure[Bus dwelling predictors time series.\label{fig:bus_series}]{
    \includegraphics[width=0.47\textwidth]{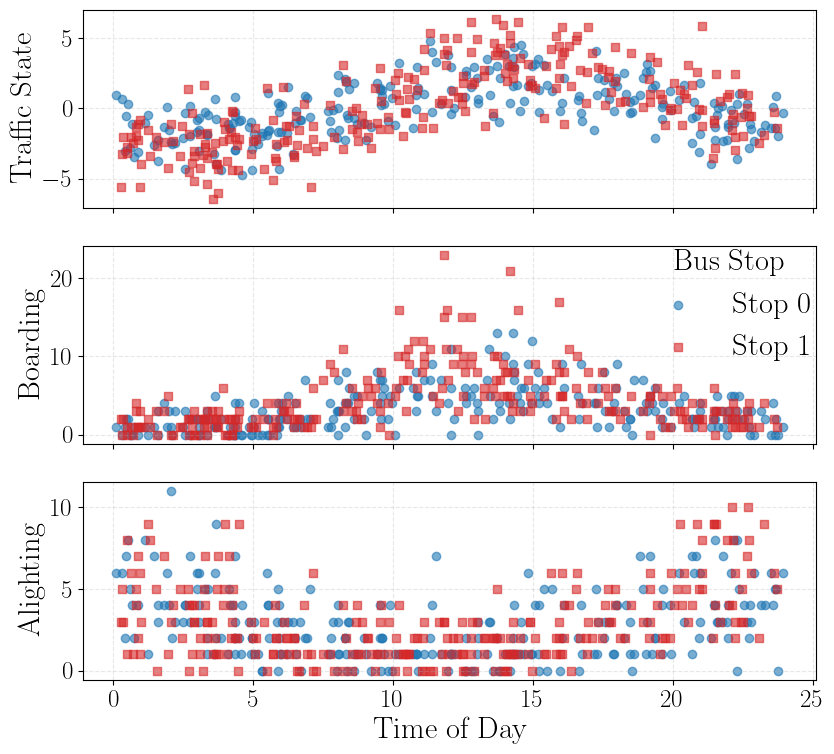}
  }

  \caption{Bus dwelling problem: (a) causal graph and (b) predictor time series examples.}
  \label{fig:combined_bus_dwelling}
\end{figure}
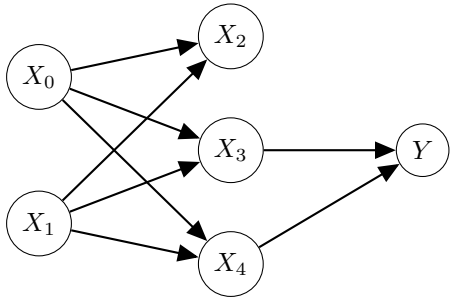
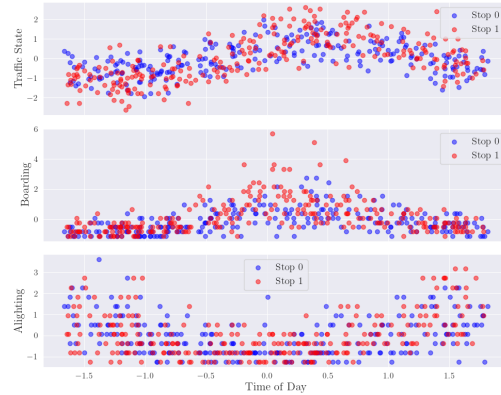

\subsubsection{Data Generation Process}

The dwelling time $Y$ is directly influenced by covariates such as the number of boarding passengers $X_3$, and the number of alighting passengers $X_4$. This choice is consistent with empirical findings that identify passenger boarding and alighting as the primary determinants of bus stop dwell time \cite{levinson1983traveltime, dueker2004dwelltime}, and here we adopt this simplified specification.
The generative model follows:
\begin{align}
Y &= f(X_3, X_4) + \epsilon,
\end{align}
where $f$ represents an underlying causal function, and $\epsilon$ is an independent noise term. Importantly, we ensure that $X_2$ (traffic state) does not directly cause $Y$, allowing us to test BHIP’s ability to avoid spurious correlations. Realistic temporal variations and passenger counts were simulated using modulated sinusoidal and Poisson processes, respectively, and are represented in Figure \ref{fig:bus_series}. Full generation details are in Appendix \ref{appendix:bus_dwelling}.

\subsubsection{Results}

We apply both BHIP and ICP to infer the set of causal predictors. Since both methods rely on tests across different environments, we define environments based on the different bus stops, that have different data generative models but the same underlying SCM. 

\begin{wraptable}{r}{0.45\textwidth}
    \vspace{-\intextsep}\centering
\caption{BHIP and ICP results}
\setlength{\tabcolsep}{4pt} 
\small 
\begin{tabular}{ cccccc }
\toprule
X & \multicolumn{2}{c}{$N=100$} & \multicolumn{2}{c}{$N=500$} \\ & ICP & BHIP & ICP & BHIP \\
\midrule
$X_0$ & - & - &  - & - \\
$X_1$  & - & - & - & - \\
$X_2$  & - & - & - & - \\
$X_3$ & \checkmark & \checkmark &  \checkmark & \checkmark \\
$X_4$ & - & - &  - & \checkmark \\
\bottomrule
\end{tabular}
\label{table:bus_dwelling_icp}

\end{wraptable}

Regarding the model setup for the experiments, we fix the observation noise variance at $\sigma^{2} = 1.0$. 
We assign standard priors to the global parameters, specifically $\mu\sim\mathcal{N}(0,1)$ and $\tau\sim\text{Half-Cauchy}(1)$, and local parameters $\beta^{e}\sim\mathcal{N}(\mu,\tau)$. 
We try two different experiments, one with the number of data points per bus stop $N=100$ and another one with $N=500$. The key question we evaluate is: can BHIP correctly recover $X_3$ and $X_4$ as causal predictors while excluding $X_0$, $X_1$ and $X_2$? 

Using non-centered parameterization (see Appendix~\ref{appendix:non_centered}), BHIP without the inclusion of informative priors identified predictors $X_3$ and $X_4$ as having significant invariant effects. For $N=500$, $X_3$ showed 100\% HDI outside ROPE for local and 95\% for global parameters ($\gamma_3=0.96$), while $X_4$ had 100\% local and 90\% global HDI outside ROPE ($\gamma_4=0.97$). Other predictors consistently showed HDIs below 65\% outside ROPE, indicating BHIP successfully identified the true causal predictors.

Table~\ref{table:bus_dwelling_icp} compares BHIP (with a decision rule: HDI out of ROPE $>85\%$, $\gamma_d > 0.85$) against ICP ($\alpha=0.05$). Covariates detected as invariant causal predictors are marked `$\checkmark$`. 
In contrast to BHIP, ICP identified $X_3$ but consistently missed $X_4$, likely due to its conservative nature (type I errors)~\citep{peters2016invariant}. Overall, BHIP demonstrated improved power in identifying both true causal predictors ($X_3, X_4$) while rejecting non-causal ones in this setting.

\subsection{Case study: Educational attainment}\label{subsec:educational}
We use the educational attainment dataset \citep{rouse1995}, which contains information on 4739 students in 1,100 high schools in the USA. The dataset includes 13 covariates such as gender, ethnicity, standardized test scores, parental education levels, family income, and the binary target variable that indicates whether a student attained a bachelor's degree or higher, corresponding to at least 16 years of education. To introduce heterogeneity, we follow the environmental split based on the distance to the nearest 4-year college, a variable assumed to have no direct causal effect on the target. Students who live closer than the median form one environment, while those living farther than the median form the other.

\subsubsection{Methodology}
We implement the BHIP framework using a hierarchical logistic regression model. The model includes environment-specific effects for each predictor while pooling information across environments through global parameters. We compare the results with those obtained using the ICP method.

For both methods, most predictors are encoded as dummy variables (e.g., \textit{fcollege\_yes} indicates if the father attained a college degree). The primary goal is to identify invariant predictors with non-zero effects across environments and quantify uncertainty in their contributions.

\subsubsection{Results}
We applied BHIP on the educational attainment dataset, with some results shown in Figure \ref{fig:score} (additional figures in Appendices \ref{appendix:figures}). While direct comparison to ICP's results is for intuition due to lack of ground truth, we hereby show how BHIP offers a richer analysis.

\begin{figure}[ht!]
    \centering
    \includegraphics[width=0.495\textwidth]{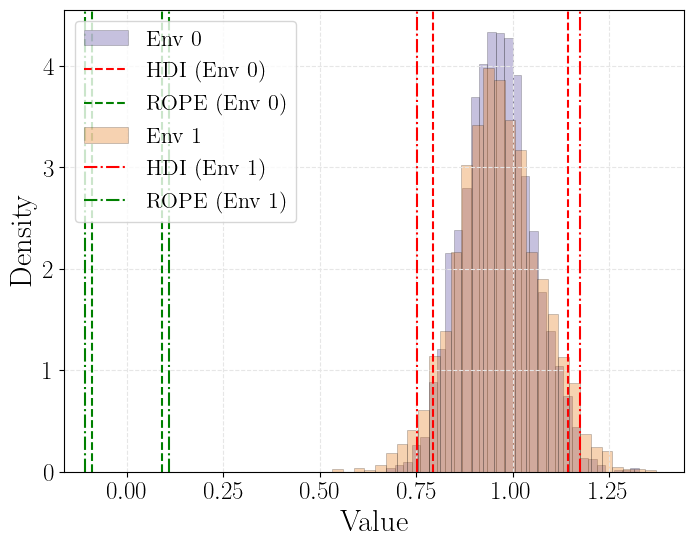}
    \includegraphics[width=0.495\textwidth]{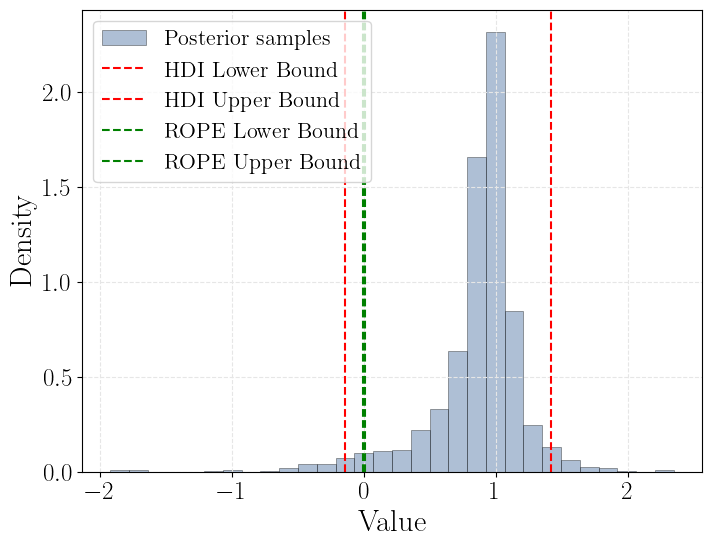}
    \caption{HDI + ROPE for local parameters left (Beta score) and global parameters right (mu score) of predictor: \texttt{score}. Pooling factor of $\gamma_d=1$.} 
    \label{fig:score}
\end{figure}

BHIP identified achievement \textit{score} as a strong invariant predictor (Figure~\ref{fig:score}), with 95\% HDI outside ROPE and $\gamma=1$, aligning with ICP findings. Father's education (\textit{fcollege\_yes}) local parameters do not lie completely outside the ROPE in any environment  but showed a positive, largely invariant effect ($\gamma=0.92$),  again similar to ICP's finding for the inverse predictor. Notably, BHIP identified \textit{income\_low} as a negative invariant predictor ($\gamma=0.99$), an effect not detected by ICP. Other predictors like \textit{region\_west} and ethnicity showed more nuanced, non-invariant effects (HDIs overlapping ROPE), consistent with ICP's exclusions. While both methods identified \textit{score} and father's education, BHIP's hierarchical modeling provided a richer  analysis of environment-specific effects and uncertainty quantification, demonstrating its flexibility.

\subsection{Computational Complexity Study}\label{subsec:computational}

We generated 100 random DAGs for configurations with the number of nodes $N$ ranging from 3 to 20. For simplicity and focus on the algorithmic scaling with $N$, each problem instance used 200 samples per environment and 2 environments. Figure \ref{fig:computational_time} displays the distribution of the computational times. As hypothesized based on the algorithmic differences, Figure \ref{fig:computational_time} shows that the computational time for ICP grows exponentially with the number of nodes. In contrast, BHIP's runtime increases at a much slower rate. This empirical result strongly supports the notion that leveraging the hierarchical Bayesian framework to assess invariance through parameter relationships, rather than relying on exhaustive conditional independence testing across all predictor subsets, allows BHIP to remain computationally viable for problems involving a larger number of potential causal predictors.

\subsection{Benchmark against invariant prediction methods}\label{subsec:bip_benchmark}

To position BHIP among recently proposed invariant prediction approaches, we adopt the experimental setups of \citet{wu2025bayesianinvariancemodelingmultienvironment}, who introduce Bayesian Invariance Modeling of Multi-Environment Data (BIP) and its variational approximation VI-BIP. In particular, we (i) replicate their low-dimensional linear-Gaussian synthetic study, where the true invariant feature set is known, and (ii) reuse the large-scale yeast gene perturbation dataset originally collected by \citet{kemmeren2014largescalegeneticperturbations} and previously analyzed in the context of invariant prediction by \citet{peters2016invariant,meinshausen2016gene}. In these experiments, in addition to BHIP (ran with HDI–ROPE decision rule and pooling factors of $0.95$), BIP and VI-BIP (ran with the recommended hyperparameters), we include ICP \citep{peters2016invariant}, which we run with significance level $\alpha = 0.05$ for its multiple hypothesis testing procedure; Hidden-ICP \citep{rothenhausler2019causaldantzig}, a relaxed invariance extension of ICP that we also run with $\alpha = 0.05$; and EILLS \citep{fan2024environmentinvariantlinear}, a linear regression method with an invariance regularization term across environments, for which we follow the default implementation and set the regularization strength to $\gamma = 36$.

\subsubsection{Synthetic benchmark}\label{subsec:bip_synth}

Following \citet{wu2025bayesianinvariancemodelingmultienvironment}, we generate multi-environment linear-Gaussian data with $p=10$ predictors. We vary the number of environments $E \in \{2,5,10,20\}$ and the per-environment sample size, and evaluate all methods using the exact discovery rate (probability of recovering the invariant set exactly) and coverage (probability that the estimated invariant set is a subset of the true set). 

To maintain consistency with the original benchmark by \citet{wu2025bayesianinvariancemodelingmultienvironment}, we focus our main presentation on exact discovery and coverage. However, to better understand traditional metrics results including precision, recall, and F1 scores are available in Appendix \ref{appendix:precision_recall}.

Figure~\ref{fig:bip_lowdim} summarizes the results. As $E$ increases, BHIP approaches perfect exact discovery, closely tracking BIP and VI-BIP and substantially outperforming ICP and Hidden-ICP, which remain either overly conservative (high coverage but low exact discovery for ICP) or less accurate overall. BHIP matched ICP with coverage of 1 irrespective of $E$, outperforming BIP and VI-BIP. Furthermore, while EILLS performs well when the number of environments is high, it exhibits low coverage in low-environment experiments, which, as detailed in the appendix, suggests low precision under those conditions.

\begin{figure}[htbp]
    \centering

    \subfigure[Computational time distribution for ICP vs.\ BHIP across varying numbers of nodes.\label{fig:computational_time}]{
        \includegraphics[width=0.415\textwidth]{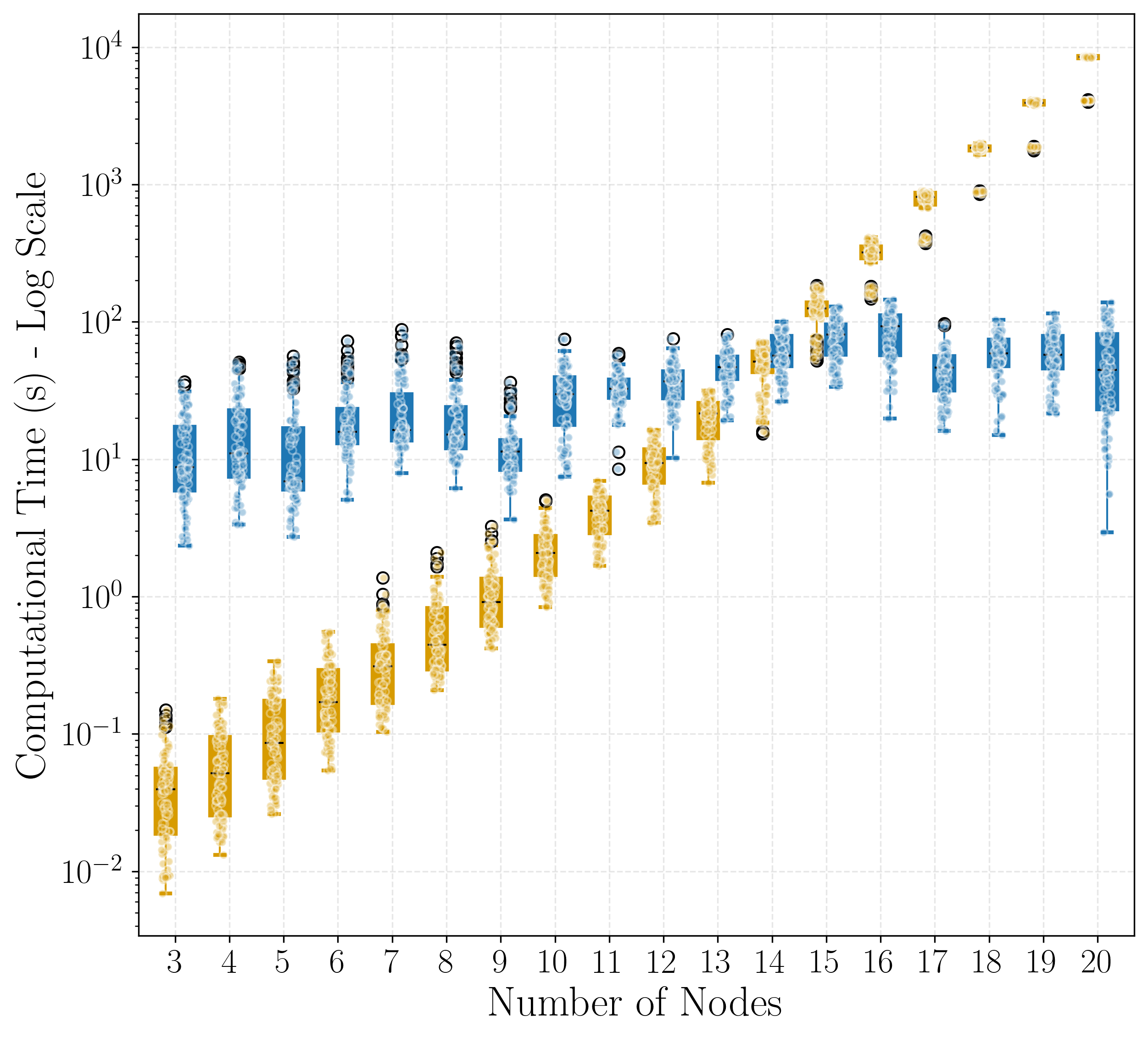}
    }
    \hfill
    \subfigure[Exact discovery (above) and coverage (below) as a function of the number of environments $E$ for BHIP and competing invariant prediction methods.\label{fig:bip_lowdim}]{
        \includegraphics[width=0.546\textwidth]{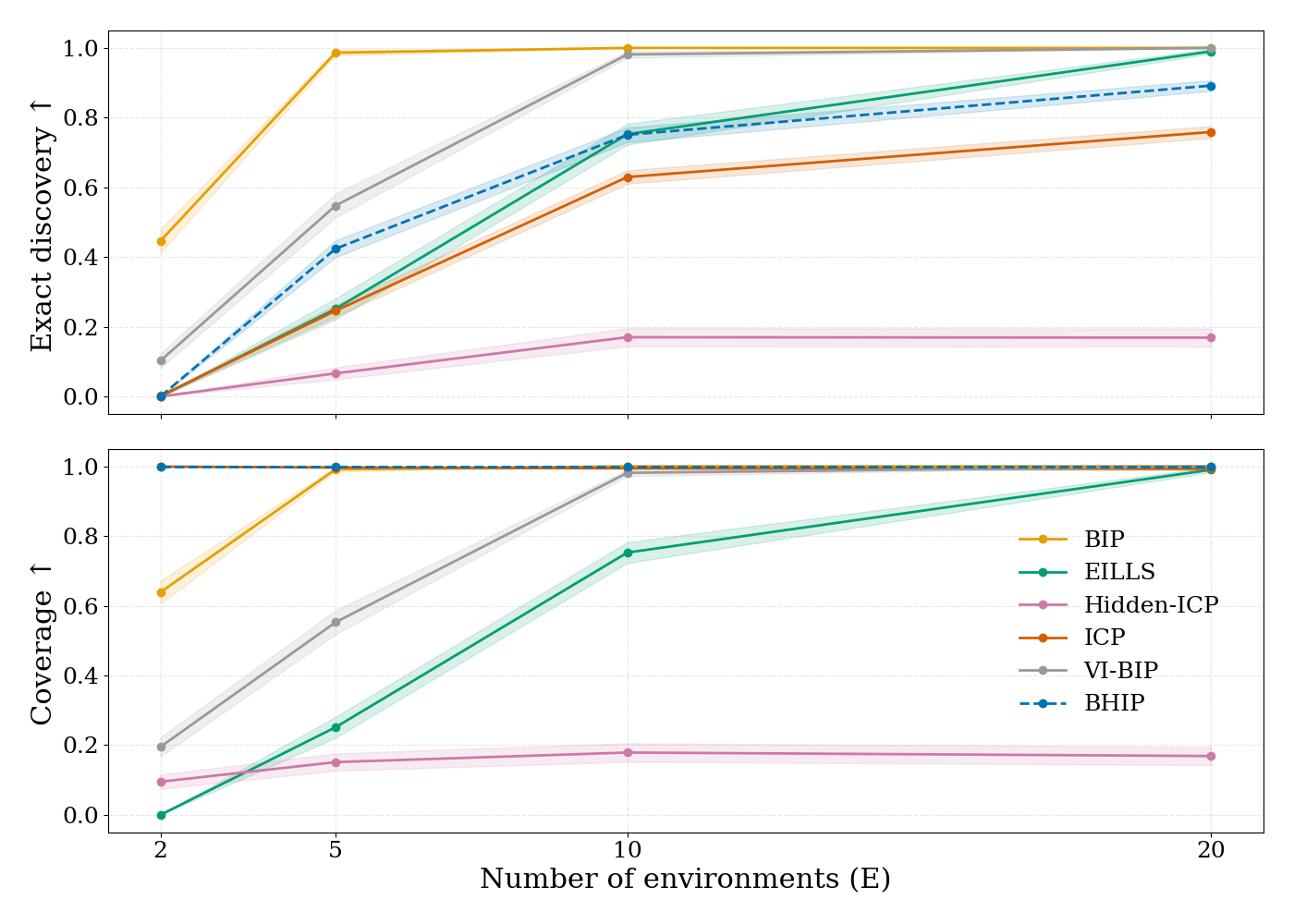}
    }

    \caption{(a) Computational cost comparison of ICP and BHIP; (b) low-dimensional synthetic benchmark of BHIP and competing invariant prediction methods.}
    \label{fig:time_and_lowdim}
\end{figure}

\subsubsection{Gene perturbation benchmark}\label{subsec:bip_real}

Finally, we consider the gene perturbation dataset from large-scale yeast gene deletion experiments \citep{kemmeren2014largescalegeneticperturbations}. The data comprise genome-wide mRNA expression levels measured under an observational condition and under single-gene deletion interventions. Following past works, we focus on a benchmark set of target genes and, for each target, compare the invariant feature genes inferred by different methods. As in \citet{wu2025bayesianinvariancemodelingmultienvironment}, ICP-s first applies a Lasso-based pre-screening that retains 10 candidate predictors per target before running ICP with significance level $\alpha = 0.01$; for BHIP we use the same screening scores but keep the top 200 candidates per target to avoid prohibitively large, memory-wise, Monte Carlo fits over all genes.

Table~\ref{tab:gene_perturbation} summarizes the inferred invariant feature genes for the 10 benchmark targets. Overall, BHIP broadly agrees with the previously validated effects: selected features mostly coincide with genes identified by ICP-s and/or VI-BIP and reported as true effects in \citet{meinshausen2016gene}. For some targets BHIP returns an empty invariant set, behaving more conservatively than VI-BIP, whereas for others it proposes slightly larger invariant sets that extend the ICP-s predictions. Taken together, these results indicate that BHIP can recover the core validated regulatory relationships in this benchmark while remaining relatively conservative in high-dimensional regions.

\begin{table}[htbp]
    \centering
    \scriptsize
    \caption{Inferred invariant feature gene(s) for benchmark target genes in the yeast perturbation dataset. $\varnothing$ denotes that no invariant feature was selected. Genes in blue are validated to have significant effects on the corresponding target gene}
    \label{tab:gene_perturbation}
    \begin{tabular*}{\textwidth}{@{\extracolsep{\fill}} l p{0.18\textwidth} p{0.18\textwidth} p{0.22\textwidth} p{0.22\textwidth} @{}}
        \toprule
        Target gene &
        \multicolumn{4}{c}{Inferred invariant feature gene(s)}\\
        \cmidrule(l){2-5}
        & \cite{meinshausen2016gene}
        & ICP-s ($\alpha = 0.01$)
        & VI-BIP ($t = 0.5$)
        & BHIP\\
        \midrule
        YMR103C
        & \textcolor{blue}{YMR104C}
        & \textcolor{blue}{YMR104C}
        & \textcolor{blue}{YMR104C}, YHR209W
        & \textcolor{blue}{YMR104C} \\
        YMR321C
        & \textcolor{blue}{YPL273W}
        & \textcolor{blue}{YPL273W}
        & \textcolor{blue}{YPL273W}
        & \textcolor{blue}{YPL273W} \\
        YCL042W
        & \textcolor{blue}{YCL040W}
        & \textcolor{blue}{YCL040W}
        & \textcolor{blue}{YCL040W}
        & \textcolor{blue}{YCL040W} \\
        YLL020C
        & \textcolor{blue}{YLL019C}
        & \textcolor{blue}{YLL019C}
        & \textcolor{blue}{YLL019C}
        & \textcolor{blue}{YLL019C} \\
        YPL240C
        & \textcolor{blue}{YMR186W}
        & \textcolor{blue}{YMR186W}
        & \textcolor{blue}{YJL077C}, \textcolor{blue}{YMR186W}
        & \textcolor{blue}{YMR186W}, YOL121C, YLL045C, YMR143W \\
        YBR126C
        & \textcolor{blue}{YDR074W}
        & $\varnothing$
        & YGR008C, \textcolor{blue}{YDR074W}, YKL035W
        & $\varnothing$ \\
        YMR173W-A
        & \textcolor{blue}{YMR173W}
        & \textcolor{blue}{YMR173W}, YOL100W
        & \textcolor{blue}{YMR173W}
        & \textcolor{blue}{YMR173W} \\
        YGR264C
        & \textcolor{blue}{YGR162W}
        & $\varnothing$
        & $\varnothing$
        & $\varnothing$ \\
        YJL077C
        & \textcolor{blue}{YOR027W}
        & $\varnothing$
        & YLL026W, \textcolor{blue}{YOR027W}, YFL010W-A
        & \textcolor{blue}{YOR027W}, YDR214W \\
        YLR170C
        & YJL115W
        & YDR322C-A, YDR180W, YJL184W, YLR438C-A
        & $\varnothing$
        & $\varnothing$ \\
        \bottomrule
    \end{tabular*}
\end{table}

In conclusion, the results reinforce the flexibility of BHIP in identifying invariant predictors while quantifying their uncertainty. Unlike ICP, BHIP provides a probabilistic perspective that is crucial for real-world applications where heterogeneity and data limitations are prevalent and allows for the use of priors whether to incorporate domain knowledge or regularization.

\section{Discussion and Conclusion}\label{sec:discussion_conclusion}
Our study presents the BHIP framework, a novel approach for identifying invariant predictors across heterogeneous environments. While BHIP builds upon established techniques such as BHMs and invariance testing, its novelty lies in its specific synthesis and targeted application to probabilistically test causal invariance for parent discovery.

Despite its significant strengths, BHIP also presents some limitations. As a Bayesian method relying on MCMC for inference, it incurs a higher computational overhead cost compared to frequentist alternatives like ICP. It also assumes that the defined environments capture meaningful heterogeneity reflective of underlying causal structure, an aspect that may not always hold perfectly in real-world applications. Moreover, while offering flexibility, the effectiveness of BHIP depends on appropriate model specification, selection of priors, and careful interpretation of the rich posterior outputs, requiring a degree of understanding of BHMs. Nonetheless, our analysis demonstrated that BHIP effectively identifies strong invariant predictors while comprehensively quantifying the uncertainty of their effects and invariance. This enables a more nuanced and informative analysis of predictors than a simple binary classification approach. Furthermore, while ICP is limited by the curse of dimensionality, BHIP scales more effectively and remains computationally tractable even with a significantly larger number of predictors $p$.

BHIP's capacity for probabilistic invariance testing, rigorous uncertainty quantification, and identification of robust causal predictors represents a significant step forward for causal discovery in complex, real-world scenarios. Future research could explore several directions. One avenue is developing more computationally efficient inference algorithms and extending the framework to capture more intricate forms of heterogeneity and non-linear causal relationships. Another is investigating how an active learning approach could leverage the extra information from BHIP to perform interventions that generate ideal environments. Overall, BHIP contributes a principled, probabilistic approach to discovering invariant causal relationships across diverse environments.

\acks{We thank Luhuan Wu for the facilitation of code and experiments relative to BIP. The work presented in this article is supported by Novo Nordisk Foundation grant NNF23OC0085356.}

\bibliography{main}


\clearpage
\begin{center}
	\Large\bfseries Supplementary Material
\end{center}
\appendix


\begin{table}[htbp]
\centering
\caption{Notation table}
\label{tab:notation_bhip}
\small
\begin{tabular}{l p{10.2cm}}
\toprule
\multicolumn{2}{l}{\textbf{Sets, indices, and environments}}\\
\midrule
$D$ & Number of variables (excluding the target $Y$).\\
$\mathcal{E}$ & Index set of environments; $e \in \mathcal{E}$ indexes an environment.\\
$N_e$ & Number of samples in environment $e$; $N \!=\! \sum_{e\in\mathcal{E}} N_e$.\\
$n$ & Index for observations within an environment, $n=1,\dots,N_e$.\\
$S \subseteq \{1,\dots,D\}$ & Subset of predictor indices; $X_S$ denotes the corresponding covariates.\\

\midrule
\multicolumn{2}{l}{\textbf{Random variables and data}}\\
\midrule
$\mathbf{X} \!=\! (X_1,\dots,X_D)$ & Predictors / covariates.\\
$Y$ & Target variable.\\
$(Y^e,\mathbf{X}^e)$ & Data in environment $e$ (cf.\ Eq.~\eqref{eq: hetero data}).\\
$y_n^e,\, X_n^e$ & $n$-th observation of target and covariates in environment $e$.\\
\midrule
\multicolumn{2}{l}{\textbf{SCM and graphical model}}\\
\midrule
$\mathcal{M}\!=\!(\mathbb{S},P_\varepsilon)$ & Structural Causal Model (SCM).\\
$X_d := f_d(\mathrm{PA}(X_d),\varepsilon_d)$ & Structural assignment for variable $X_d$.\\
$\mathcal{G}$ & DAG induced by the SCM over $(\mathbf{X},Y)$.\\
$\mathrm{PA}(Y)$ & Set of (direct) causal parents of $Y$.\\
\midrule
\multicolumn{2}{l}{\textbf{Distributions}}\\
\midrule
$\mathcal{N}(\cdot,\cdot)$, $\mathsf{Cauchy}^+(\cdot,\cdot)$ & Normal and half-Cauchy distributions.\\
\midrule
\multicolumn{2}{l}{\textbf{Hierarchical model parameters (BHIP)}}\\
\midrule
$\boldsymbol{\beta}^e \!=\! (\beta_1^e,\dots,\beta_D^e)$ & Environment-specific regression coefficients.\\
$\phi_d$ & Global (population) parameter(s) governing $\beta_d^e$ (generic notation).\\
$\mu_d,\, \tau_d$ & Mean and standard deviation hyperparameters for $\beta_d^e$ (e.g., $\beta_d^e\!\sim\!\mathcal{N}(\mu_d,\tau_d^2)$).\\
$\sigma^2$ & Observation noise variance in Gaussian likelihoods.\\
$\delta_d^e$ & Deviation from the global mean: $\delta_d^e \!=\! \beta_d^e - \mu_d$.\\
\midrule
\multicolumn{2}{l}{\textbf{Bayesian testing / decision quantities}}\\
\midrule
$\gamma_d$ & Pooling factor for predictor $d$:
$\displaystyle \gamma_d \!=\! 1 - \frac{\mathrm{Var}_{e}\!\left[\,\mathbb{E}[\delta_d^e]\,\right]}{\mathbb{E}\!\left[\,\mathrm{Var}_{e}(\delta_d^e)\,\right]}$.\\
$\mathrm{HDI}_{q}(\theta)$ & $q\%$ Highest Density Interval of a posterior for parameter $\theta$.\\
$\mathrm{ROPE}$ & Region of Practical Equivalence around $0$ (e.g., $[-\epsilon,\epsilon]$).\\
$p_{\mathrm{out}}(\theta)$ & Posterior mass of $\theta$ outside ROPE (used to judge non-zero effects).\\
\midrule
\multicolumn{2}{l}{\textbf{Sparsity priors}}\\
\midrule
$\lambda_d$ & Local shrinkage (horseshoe) for predictor $d$; $\lambda_d\!\sim\!\mathsf{Cauchy}^+(0,1)$.\\
$\tau$ & Global shrinkage (horseshoe); $\tau\!\sim\!\mathsf{Cauchy}^+(0,1)$.\\
$z_d$ & Spike-and-slab inclusion indicator for predictor $d$ ($z_d\!\in\!\{0,1\}$).\\
$\pi$ & Prior inclusion probability for spike-and-slab ($z_d\!\sim\!\mathrm{Bernoulli}(\pi)$).\\

\bottomrule
\end{tabular}\label{appendix:notation}
\end{table}

\clearpage
This supplementary material is organized as follows. 
Table~\ref{appendix:notation} summarizes the notation used throughout the paper and 
Appendix~\ref{appendix:assumptions} states the key causal and statistical assumptions underlying BHIP. 
Appendices~\ref{appendix:proposition} and~\ref{appendix:guarantees} provide theoretical support for the invariance principle and the behavior of the pooling factor. 
Appendix~\ref{appendix:non_centered} describes the non-centered parameterization and inference scheme, while Appendix~\ref{subsec:sparse} details the use of sparse priors. 
Appendices~\ref{appendix:bus_dwelling}–\ref{tu_results} contain additional experimental details and results for the Bus Dwelling Problem, the educational attainment study, the large-scale synthetic simulations, and the TU Danish Travel Survey.

\section{Key Assumptions}\label{appendix:assumptions}

This section outlines the key assumptions required for the BHIP framework. These assumptions stem from those commonly made in Causal Discovery, while also including assumptions specific to our Bayesian modeling approach.

BHIP relies on the following assumptions:

\textbf{Causal and Graphical Model Assumptions}

These assumptions are standard in many causal discovery methods, including ICP.
\begin{itemize}
    \item \textbf{Acyclicity:} The underlying causal system relating the variables $\mathbf{X}$ and $Y$ can be represented by a DAG.
    \item \textbf{SCM Representation:} The data is assumed to be generated by an underlying SCM, where each variable is a function of its direct causes (parents) and an independent noise term \citep{pearl_causal_2009}.
    \item \textbf{Faithfulness}: The conditional independence relationships observed in the data from each environment precisely correspond to the d-separation properties in the true causal graph, and vice versa \citep{spirtes1993causation}.
    \item \textbf{Causal Sufficiency (relative to $\mathbf{X}, Y$):} There is no unobserved confounding between the predictor variables $\mathbf{X}$ and the target variable $Y$ that would create spurious invariant associations across the observed environmental changes.
\end{itemize}

\textbf{Invariance Assumptions}

These assumptions are central to the principle of identifying causal relationships by leveraging heterogeneity across environments, as in ICP.
\begin{itemize}
    \item \textbf{Heterogeneous Environments:} The data is observed across distinct environments $e \in \mathcal{E}$, where the distributions of the predictor variables $P(\mathbf{X}^e)$ are allowed to change.
    \item \textbf{Structural Invariance (Assumption 2):} The conditional distribution of the target variable given its true causal parents, $P(Y | PA(Y))$, remains invariant across all environments $e \in \mathcal{E}$. The distribution of the noise term $\varepsilon_Y$ for the target variable is also assumed to be the same across environments.
\end{itemize}

\textbf{BHIP Statistical Model and Inference Assumptions}

These assumptions pertain specifically to the statistical modeling choices and inference procedure employed by BHIP.
\begin{itemize}
    \item \textbf{Correct Likelihood \& Functional Form:} The chosen likelihood function (e.g., Gaussian for continuous $Y$, Logistic for binary $Y$) and the functional form of the relationship between $Y$ and its predictors within each environment (e.g., linear $Y \approx \mathbf{X}\beta^e$) accurately model $P(Y | \mathbf{X}, \beta^e)$. While the model is presented with a linear form, the framework can be extended to handle non-linear relationships, similar to extensions in the frequentist ICP  \citep{heinze2018invariant}.
    \item \textbf{Correct Hierarchical Structure:} The assumption that the environment-specific coefficients $\beta_d^e$ for a given predictor $d$ are drawn from a common distribution (e.g., $\text{Normal}(\mu_d, \tau_d^2)$) across environments adequately captures the underlying structure of parameter variation and sharing.
    \item \textbf{Reasonable Priors:} The prior distributions selected for the model parameters are chosen such that they allow the posterior distribution to concentrate correctly on the true parameter values, particularly in the large sample limit. Sparsity-inducing priors (like Horseshoe or Spike-and-Slab) are assumed to appropriately reflect beliefs about the sparsity of the true parent set and facilitate variable selection.
    \item \textbf{Independent Noise:} The noise term $\varepsilon_Y$ is assumed to be independent across individual observations and across environments (conditional on the parent variables).
    \item \textbf{MCMC Convergence:} The Markov Chain Monte Carlo (MCMC) algorithm used for posterior inference is assumed to have converged, providing samples that accurately represent the true posterior distribution. 
\end{itemize}

\textbf{Assumptions for Asymptotic Guarantees}

These conditions are typically required for theoretical results regarding the asymptotic behavior and consistency of Bayesian estimators and variable selection procedures.
\begin{itemize}
    \item \textbf{Standard Regularity Conditions:} Standard technical conditions necessary for Bayesian asymptotic theorems (such as the Bernstein-von Mises theorem or results on posterior consistency for hierarchical and sparse models) are assumed to hold for the specific model and data.
    \item \textbf{Sufficient Heterogeneity:} The changes in the predictor distributions $P(\mathbf{X}^e)$ across environments must be sufficiently diverse and informative to allow the true invariant parent set $S^*$ to be reliably distinguished from non-invariant predictors.
    \item \textbf{Appropriate Thresholds:} The decision thresholds used for variable selection based on posterior summaries (e.g., ROPE sizes for HDI tests, the threshold value for the pooling factor) are assumed to be set at appropriate values relative to the scale of the true effects and the rate of posterior concentration.
\end{itemize}

Note on \textbf{Identifiability}:

Given the aforementioned assumptions the true invariant parent set $\pa(Y)$ is identifiable from the data. BHIP's statistical framework is designed to leverage this identifiability property and recover $\pa(Y)$
  from the data by identifying variables whose relationship with $Y$ is invariant across the observed environments.

\section{Invariance Proposition and Proof}\label{appendix:proposition}

Full proposition and proof \citep{buhlmann2018}:
\begin{prop}\label{prop:inv} 
Assume an SCM as in Assumption \ref{assump:strinv} and let the set of environments $\mathcal{E}$ be such
that Assumption \ref{assump:strinv} holds. Then, the set of direct causes $\pa(Y)$ is invariant with respect
to $\mathcal{E}$ and Assumption \ref{assump:inv} holds.
\end{prop}

\begin{proof} 
 In Assumption \ref{assump:strinv} it is seen that the conditional distribution of $Y^e$ given $X^e_{\pa(Y)}$ is
 fully determined by the structural equation $f_Y$ and the noise distribution $F_\mathcal{E}$, both remaining the same for all $e \in \mathcal{E}$, thus Assumption \ref{assump:inv} is satisfied and $\pa(Y)$ is invariant with respect to $\mathcal{E}$. 
 \end{proof} 

\section{Asymptotic Guarantees}\label{appendix:guarantees}

BHIP's ability to recover the true invariant parent set $\pa(Y)$ asymptotically (as sample size $N \to \infty$) is motivated by Bayesian posterior consistency. While BHIP's specific posterior-based decision rules differ from ICP's frequentist tests, we argue for its asymptotic correctness. Under standard regularity conditions, Bayesian posterior distributions concentrate around true parameter values, often approximating normality centered at efficient estimators (as suggested by the Bernstein-von Mises (BvM) theorem \citep{vonMises}). We expect this concentration for BHIP's hierarchical parameters $(\mu_d, \tau_d, \beta_{d}^{e})$. Furthermore, the use of appropriate sparsity priors yields posterior consistency for variable selection in high-dimensional regression settings. This expected posterior concentration implies specific asymptotic behavior for BHIP's tests:

\begin{itemize}
    \item \textbf{HDI+ROPE Tests:} For true parents $d \in \pa(Y)$ (with non-zero effects), the concentrating posteriors for $\mu_d$ and $\beta_{d}^{e}$ will eventually place the HDIs entirely outside a fixed ROPE, leading to inclusion with probability approaching 1. For non-parents with zero effects, the HDIs will concentrate within the ROPE, ensuring exclusion.
    \item \textbf{Pooling Factor:} For true parents $d \in \pa(Y)$, invariance (Assumption \ref{assump:inv}) implies identical $\beta_{d}^{e}$. Posterior concentration on this common value should lead the pooling factor $\gamma_d$ posterior to concentrate near $1$. For non-invariant non-parents, differing true $\beta_{d}^{e}$ will keep the pooling factor below asymptotically. See Proof \ref{proof:consistency} below.
\end{itemize}

The combination of asymptotic posterior concentration, consistent sparsity priors, and the logic of the invariance tests provides strong theoretical motivation for BHIP's ability to correctly identify $\pa(Y)$ asymptotically under the stated assumptions.

\subsection{Pooling Factor Consistency}\label{proof:consistency}

Let $X_1, X_2, \dots, X_n$ be i.i.d. observations from a parametric model $\{P_\theta : \theta \in \Theta\}$, where $\Theta\subseteq \mathbb{R}^k$ and the true parameter $\theta^\ast\in\Theta$.

For a Bayesian approach, the parameter $\theta$ is treated as random and encodes our prior beliefs about the value of the parameter in a distribution $\pi$.
The setup is:
\begin{align*}
\theta &\sim \pi, \\
\{X_1, \ldots, X_n\} \mid \theta &\sim P_\theta.
\end{align*}
Using Bayes'  Theorem the posterior distribution of $\theta$ can be computed as:
\begin{align}
p(\theta \mid X_1, \ldots, X_n) = \frac{\mathcal{L}(\theta; X_1, \ldots, X_n) \pi(\theta)}{\int_\theta \mathcal{L}(\theta; X_1, \ldots, X_n) \pi(\theta)d\theta} \propto \mathcal{L}(\theta) \pi(\theta).
\end{align}

The Bernstein–von Mises (BvM) theorem guarantees us that in fixed-dimensional problems, under the assumption that the prior is continuous and strictly positive in a neighborhood around $\theta^*$, the posterior distribution converges in total-variation distance to a Gaussian distribution centered around the Maximum Likelihood Estimator (MLE) $\hat{\theta}_n$, that is,

\begin{align}
\left\lVert p(\theta \mid X_1, \ldots, X_n) - \mathcal{N}\left( \hat{\theta}_n, \frac{1}{n} \mathcal{I}(\hat{\theta}_n)^{-1} \right)\right\rVert_{\text{TV}} \to 0,
\end{align}

where $\mathcal{I}(\hat{\theta}_n)$ is the Fisher information matrix around the MLE estimate. 
Assuming a regular parametric model and that the BvM theorem holds, the posterior distribution converges in probability to a Gaussian distribution around the true parameter with rapid shrinkage, that is, as $n\to \infty$

\begin{align}
p_n(\theta \mid X^{(n)}) 
\xrightarrow{\text{TV}} 
\mathcal{N}\left( \hat{\theta}_n,\; \frac{1}{n} I^{-1}(\theta^\ast) \right)
\end{align}

which implies that the posterior is a consistent estimator for $\theta^\ast$ as
\begin{align}
\mathbb{E}[\theta \mid X^{(n)}] =\int \theta dp_n(\theta\mid X^{(n)})
\xrightarrow{P} 
\theta^\ast
\end{align}

with variance shrinking at a rate of $1/n$
\begin{align}
\operatorname{Var}(\theta \mid X^{(n)}) 
\xrightarrow{P} 
\frac{1}{n} I^{-1}(\theta^\ast)
\end{align}

hence as $n \to \infty$, we have

\begin{align}
\operatorname{Var}(\theta \mid X^{(n)}) 
\xrightarrow{P}
0
\end{align}

Consider the hierarchical Bayesian BHIP model where for each covariate $X_d$, $d=1,\dots,D$ we have:
\begin{align*}
\phi &\sim \pi \quad &&\text{(global parameter)} \\
\beta^e \mid \phi &\sim p(\beta^e \mid \phi) \quad &&\text{(parameter in environment e)} \\
\mathbf{X}^e \mid \beta^e &\sim p(\mathbf{X}^e  \mid \beta^e) \quad &&\text{(observed data for environment e)},
\end{align*}
where $e\in \mathcal{E}$ is environment with $|\mathcal{E}|=E$, $\mathbf{X}^e$ is the observed data and we denote the size $N^e$ and $N=\sum_{e\in\mathcal{E}} N^e$.

The \textbf{Pooling Factor} quantifies information sharing (pooling) across environments in a BHM, \citep{gelman2006pardoe}.
Define for each environment and each covariate $X_d$ the error term $\delta_d^e$, which is the difference between the global and the local parameters, that is $\beta_d^e = \phi_d + \delta_d^e$. For each predictor $X_d$, the pooling factor is defined as,

\begin{align}
    \gamma_d = 1 - \frac{\operatorname{Var}_{e\in\mathcal{E}} \left[\mathbb{E} \left[\delta_d^e\right]\right]}{\mathbb{E}\left[\operatorname{Var}_{e\in\mathcal{E}} \left[\delta_d^e\right] \right]}.
\end{align}

\begin{lemma}
Assume the BvM theorem holds for both $\phi$ and $\beta^e$ for all $e\in \mathcal{E}$, then their deviations $\beta^e-\phi$ concentrates around the true difference $\beta^{e^\ast}-\phi^\ast$ with uncertainty shrinking to zero as $N^e, E \to \infty $.
\label{lemma1}
\end{lemma}

\begin{proof*} Assume the BvM theorem holds for both $\phi$ and $\beta^e$ (we need sufficient data per environment $N^e \to \infty$ and sufficient number of environments $E \to \infty$, regular parametric model, etc), then 
the posterior distributions are asymptotically normal:

\begin{align}
p(\phi \mid \mathbf{X}) &\to \mathcal{N}\left( \hat{\phi}, \frac{1}{E\cdot N^e} I_\phi^{-1}(\phi^\ast) \right) \\
p(\beta^e \mid \mathbf{X}^e) &\to \mathcal{N}\left( \hat{\beta}^e, \frac{1}{N^e} I_{\beta^e}^{-1}(\beta^{{e^\ast}}) \right)
\end{align}
With $\hat{\phi},\hat{\beta}^e$ being consistent estimators for the true parameters $\phi^\ast,\beta^{e^\ast}$ and $I_\phi(\phi^\ast)$ the Fisher information for $\phi$, based on the marginal model $p(\mathbf{X} \mid \phi)$ and $I_{\beta}(\beta^{e^\ast})$ the Fisher information for $\beta^e$, from the environment-specific likelihood $p(\mathbf{X}^e \mid \beta^e)$.

Under BvM the posterior distributions of $\hat{\phi}$ and $\hat{\beta}^e$ each become independent Gaussians in the limit, so their difference is also approximately Gaussian, with variance equal to the sum of variances, that is

\begin{align}
    p(\beta^e - \phi \mid x) \to \mathcal{N}\left( \hat{\beta}^e - \hat{\phi},\; \frac{1}{N^e} \left(I_{\beta^e}^{-1} + \frac{1}{ E} I_\phi^{-1}\right) \right).
\end{align}
Thus, as $N^e,E\to \infty$ the posterior for the difference between local and global parameters converges to a degenerate distribution at the true difference $\beta^{e^\ast}-\phi^\ast$.  
\end{proof*}

\begin{lemma}[Asymptotic behavior of the posterior deviation ratio]
Assume a BHIP model. Under regularity conditions and the BvM theorem applied to $\phi$ and $\beta^e$ then the posterior deviation ratio converges to
\begin{align}
R 
&:= \frac{
\operatorname{Var}_{e}\left( \mathbb{E}[\delta^{e} \mid X] \right)
}{
\mathbb{E}\left[ \operatorname{Var}_{e}(\delta^{e}) \mid X \right]
}
\xrightarrow{P}
\begin{cases}
1 & \text{if } \operatorname{Var}_e(\beta^{e*} - \phi^*) > 0 \\
0 & \text{if } \beta^{e*} = \phi^* \text{ for all } e
\end{cases}
\end{align}
as $N\to \infty$ and optionally as $E\to \infty$.
\label{lemma2}
\end{lemma}
\begin{proof*}
The numerator $\operatorname{Var}_{e}\left( \mathbb{E}[\delta^{e} \mid X] \right)$ is the posterior-estimated between-environment variance in deviations and the denominator is the total posterior variance across environments. From the law of total variance we have that
\begin{align*}   
\mathbb{E}\left[ \operatorname{Var}_{e}(\beta^{e} - \phi) \mid X \right]=\mathbb{E}_e\left[ \operatorname{Var}_{e}(\beta^{e} - \phi) \mid X \right] + \operatorname{Var}_{e}\left( \mathbb{E}[\beta^{e} - \phi \mid X] \right) 
\end{align*}
where $\mathbb{E}_e\left[ \operatorname{Var}_{e}(\beta^{e} - \phi) \mid X \right]$ is the posterior within-environment variance, so the ratio R is the proportion of posterior variability in the deviations $\delta^e$ that comes from differences in posterior means across environments.

Under regularity conditions and the BvM theorem applied to $\phi$ and $\beta^e$, then as for Lemma \ref{lemma1}, the deviation is a consistent estimator with shrinking uncertainty, so

\begin{align}
\mathbb{E}[\beta^{e} - \phi \mid X ] &\xrightarrow{P} \beta^{e*} - \phi^* \\
\operatorname{Var}(\beta^{e} - \phi \mid X ) &\xrightarrow{P} 0
\end{align}
as $N\to \infty$ and optionally $E\to \infty$. Consequently:

\begin{align}
\operatorname{Var}_e\left( \mathbb{E}[\beta^{e} - \phi \mid X ] \right)
&\xrightarrow{P}
\operatorname{Var}_e\left( \beta^{e*} - \phi^* \right) \\
\mathbb{E}_e\left[ \operatorname{Var}(\beta^{e} - \phi \mid X ) \right]
&\xrightarrow{P} 0.
\end{align}
Lets consider the limit of the ratio for two cases:
\begin{align*}
   \lim_{N,E \to \infty} R= \frac{
\operatorname{Var}_{e}(\beta^{e*} - \phi^*)
}{
\operatorname{Var}_{e}(\beta^{e*} - \phi^*) +
\lim_{N,E \to \infty} \mathbb{E}_{e}\left[ \operatorname{Var}(\beta^{e} - \phi \mid x) \right]
}  
\end{align*}

\begin{itemize}
    \item[(i)]{Heterogeneous true effects:}\\
There exist environments $e,e'$ such that $\beta^{e*} \neq \beta^{e'*}$, and $\operatorname{Var}_{e}(\beta^{e*} - \phi^*)>0$ which implies that
\begin{align}
\lim_{N,E \to \infty} R = 1
\end{align}
 \item[(ii)]{Homogeneous true effects}\\
For all $e\in \mathcal{E}$ we have that $\beta^{e*} = \phi^*$ so $\operatorname{Var}_{e}(\beta^{e*} - \phi^*)=0$ then,
\begin{align}
\lim_{N,E \to \infty} R = 0
\end{align}
\end{itemize}
Hence the posterior deviation ratio, converges as $N\to \infty$ and optionally as $E\to \infty$ to
\begin{align}
R= \frac{
\operatorname{Var}_{e}\left( \mathbb{E}[\beta^{e} - \phi \mid X] \right)
}{
\mathbb{E}\left[ \operatorname{Var}_{e}(\beta^{e} - \phi \mid X) \right]
}
\xrightarrow{P}
\begin{cases}
1 & \text{if } \operatorname{Var}_e(\beta^{e*} - \phi^*) > 0 \\
0 & \text{if } \beta^{e*} = \phi^* \text{ for all } e
\end{cases}
\end{align}
\end{proof*}

Our main result arises as a Corollary to \Cref{lemma2}:

\begin{thm*}(Asymptotic behavior of pooling factor) 
Suppose the invariance assumption (Assumption \ref{assump:inv}) and the Bernstein von Mises holds for a BHIP model, then
\begin{align*}    
\gamma_d = 1 - \frac{\operatorname{Var}_{e\in\mathcal{E}} \left[\mathbb{E} \left[\delta_d^e\right]\right]}{\mathbb{E}\left[\operatorname{Var}_{e\in\mathcal{E}} \left[\delta_d^e\right] \right]}
\xrightarrow{P}
\begin{cases}
1 & \text{if } \beta^{e*}_d = \beta^{e'*}_d \text{ for all } e,e'\in \mathcal{E}  \\
0 & \text{otherwise.} 
\end{cases}
\end{align*}
\end{thm*}

\section{Non-centered parameterization and inference} \label{appendix:non_centered}

 Bayesian inference often uses \textit{Markov Chain Monte Carlo (MCMC)} methods to explore the posterior distribution of the model.  When sampling from a hierarchical parameter space, the convergence of MCMC methods can depend crucially on the parameterization of unknown quantities \citep{noncentered}. 

   A typical parameterization of the hierarchical model is the \textit{non-centered parameterization}, capturing the interdependent relationship between the global and local latent parameters, and can be useful in moderating degeneracies inherent to sampling from a hierarchical latent parameters space \citep{Betan}. 

 The normal probability density functions are closed under translation and scaling, so in a normal hierarchical model, the non-centered parameterization generates the local-level parameters from a parameterization of an independent standard normally distributed parameter $\eta^e \sim\mathcal{N}(0,1)$, where $\beta^e$ then is deterministically reconstructed as
 \begin{align}
     \beta^e = \mu + \tau \cdot \eta^e \sim \mathcal{N}(\mu,\tau)
 \end{align}

We implement the BHM with the aforementioned non-centered parameterization using NumPyro \citep{numpyro} and No-U-turn Hamiltonian Monte Carlo sampler (NUTS) \citep{nuts}. NumPyro provides a backend for Pyro, developed by \citep{pyro}, powered by JAX, developed by \citep{jax}, and is thus a lightweight probabilistic programming library.
To efficiently sample both continuous and discrete parameters when using the spike-and-slab prior, we combine NUTS with a Gibbs sampler \citep{casella1992gibbs}, which iteratively samples inclusion indicators $z_d$ and coefficients $\beta_d^e$. This improves convergence and computational efficiency.

\section{Sparse Priors: Horseshoe and Spike-and-Slab}\label{subsec:sparse}

Using a BHM enables us to take advantage of adjusting the priors of our model, for instance, to encourage sparsity \citep{georgeApproachesBayesianVariable1997}, improving the flexibility of this work. We consider two widely used priors in Bayesian variable selection to illustrate the flexibility of BHIP: the horseshoe prior and the spike-and-slab prior.

\textbf{Horseshoe Prior} \citep{carvalho2010horseshoe}, is a continuous shrinkage prior particularly well suited for situations where a few predictors have strong effects while many are close to zero. It is defined as:
\begin{align}
    &\beta_d^e \sim \mathcal{N}(0, \lambda_d^2 \tau^2),\\
    &\lambda_d \sim \text{Cauchy}^+(0,1),\\
    &\tau \sim \text{Cauchy}^+(0,1) 
\end{align}

The shrinkage parameter $\lambda_d$ allows large effects to escape strong shrinkage while small effects are pushed towards zero. The global scale $\tau$ controls overall sparsity, ensuring that only a few predictors significantly influence $Y$.

\textbf{Spike-and-Slab Prior} \citep{first_spike, ishwaran2005spikeandslab} is a discrete mixture model that explicitly models sparsity. It assumes that each coefficient $\beta_d$ is drawn from either a spike distribution represented by a Dirac delta function $\delta_0$ centered around zero (for irrelevant predictors) or a slab distribution, allowing nonzero values, following a Normal distribution (for relevant predictors):
\begin{align}
    z_d &\sim \text{Bernoulli}(\pi),\\
    \beta_d^e \mid z_d &= z_d \mathcal{N}(\mu_d, \tau_d^2) + (1 - z_d) \delta_0,\\
    \tau_d &\sim \text{Cauchy}^+(0,1) 
\end{align}
The binary inclusion variable $z_d$ determines whether the predictor is relevant ($z_d = 1$) or not ($z_d = 0$). The prior probability $\pi$ controls the expected proportion of relevant predictors. 
In practice, when strict adherence to theoretical conditions is paramount, one can use a continuous spike-and-slab prior \cite[Example 2]{ishwaran2005spikeandslab}. These priors achieve the same goal by replacing the point mass with a continuous, highly concentrated distribution (for example, a Gaussian with near-zero variance), thereby satisfying the necessary regularity conditions. This is, in fact, what was used in all our results since it allows for more efficient MCMC sampling. 

\subsection{Sparse Priors on Bus Dwelling Problem}

Table~\ref{table:sparse_priors} compares BHIP (with sparse priors and decision rule: HDI out of ROPE $>85\%$, $\gamma_d > 0.85$) against ICP ($\alpha=0.05$). Covariates detected as invariant causal predictors are marked `$\checkmark$`. \begin{wraptable}{r}{0.55\textwidth}
    \vspace{-\intextsep}\centering
\caption{Sparse priors and ICP results}
\setlength{\tabcolsep}{4pt} 
\small 
\begin{tabular}{ cccccccccc }
\toprule
X & \multicolumn{4}{c}{$N=100$} & \multicolumn{4}{c}{$N=500$} \\
& $z_d$ & $\lambda_d$ & ICP & BHIP & $z_d$ & $\lambda_d$ & ICP & BHIP \\
\midrule
$X_0$ & 0.03 & 0.50 & - & - & 0.00 & 0.43& - & - \\
$X_1$ & 0.00 & 0.24 & - & - & 0.00 & 0.16& - & - \\
$X_2$ & 0.09 & 0.63 & - & - & 0.00 & 0.37 &- & - \\
$X_3$ & 0.89 & 7.32 & \checkmark & \checkmark & 1.00 & 12.50 & \checkmark & \checkmark \\
$X_4$ & 0.57 & 2.36 & - & - & 1.00 & 4.64& - & \checkmark \\
\bottomrule
\end{tabular}
\label{table:sparse_priors}
    \vspace{-\intextsep}

\end{wraptable}BHIP with \textbf{spike-and-slab priors} showed increasing confidence in selecting true predictors $X_3$ and $X_4$ (inclusion probabilities $z_3, z_4 \to 1$) as $N$ increased from 100 to 500, while correctly excluding others. Similarly, the \textbf{horseshoe prior} yielded higher shrinkage parameters ($\lambda_d$) for $X_3$ and $X_4$ with $N=500$, indicating stronger effects, compared to negligible effects for $X_0, X_1, X_2$.

\subsection{Sparse Priors on School Attainment Dataset}

\begin{wraptable}{r}{0.4\textwidth}
    \vspace{-\intextsep}

\centering
\caption{Sparse priors results}
\begin{tabular}{ lcc } 
\toprule
Predictor & $z_d$ & $\lambda_{d}$  \\
\midrule
score & 1.00 & 3.40 \\ 
unemp  & 0.00 & 0.42 \\ 
wage  & 0.00 & 0.59 \\ 
tuition  & 0.00 & 0.42 \\ 
gender\_male  & 0.00 & 0.71 \\ 
ethnicity\_hispanic  & 0.00 & 2.75\\ 
ethnicity\_other & 0.00 & 0.73  \\ 
fcollege\_yes  & 1.00 & 3.91 \\ 
mcollege\_yes  & 0.00 & 2.66 \\ 
home\_yes & 0.00 & 3.01  \\ 
urban\_yes & 0.00 & 0.83 \\ 
income\_low & 0.00 & 0.75 \\ 
region\_west & 0.00 & 0.82 \\ 
\bottomrule
\end{tabular}
\label{table:school_sparse}
    \vspace{-\intextsep}
\end{wraptable}
\textbf{With sparse priors} similar results are achieved as can be seen on Table \ref{table:school_sparse}, $z_d$ values show that predictors \textit{score} and \textit{fcollege\_yes} are selected.  Similarly, then using the horseshoe prior, the two predictors that have the strongest effects are again \textit{score} and \textit{fcollege\_yes}.

\section{Bus Dwelling Problem: Data Generation}\label{appendix:bus_dwelling}
To introduce realistic temporal variations, we employ sinusoidal functions for both daily and weekly patterns:
The daily pattern captures the daily fluctuations in boarding and alighting passengers over the course of a day as well as the traffic state. It is parameterized by the peak hour, which determines the time of day when passenger activity is highest and the amplitude which controls the magnitude of daily variation.
The weekly pattern reflects weekly variations in traffic conditions. It is parameterized by the peak day which specifies the day of the week with the highest traffic impact and the amplitude that controls the magnitude of weekly variation.
These patterns allow us to simulate realistic fluctuations in passenger activity and traffic conditions over time.
Each bus stop in our experiment is characterized by specific parameters that influence dwell time:

\textbf{Traffic State, $X_2$}: Simulates traffic conditions at a given bus stop, including daily and weekly fluctuations. It does not directly cause dwell time.

\textbf{Boarding, $X_3$}, and \textbf{Alighting, $X_4$}, \textbf{Passengers}: Modeled as Poisson random variables, these parameters capture the stochastic nature of passenger arrivals and departures at a bus stop. Daily and weekly patterns modulate their base rates and peak times.

On a randomized setup, the aforementioned predictors can be represented by the data represented in Figure \ref{fig:bus_series} with $N=500$ for each bus stop. Using the parameters described above, dwell time $Y$ is generated as a function of boarding passengers $X_3$ and alighting passengers $X_4$. The function $f(X_3, X_4)$ incorporates coefficients that scale with the number of boarding and alighting passengers.  The noise term $\epsilon$ represents independent noise in dwell time, capturing unpredictable factors not accounted for by the model.

This setup ensures that we can systematically evaluate BHIP's ability to identify and distinguish causal factors influencing bus dwell time from observational noise and spurious correlations.

\section{Educational Attainment Additional Results}\label{appendix:figures}

This section serves to present additional figures of the Bayesian Hierarchical inference on the educational attainment case study.

\begin{figure}[ht!]
    \centering
    \includegraphics[width=0.495\textwidth]{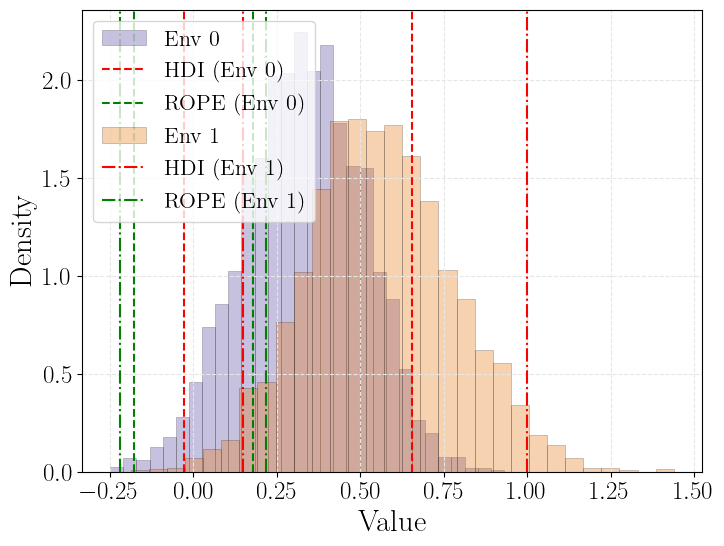}
    \includegraphics[width=0.495\textwidth]{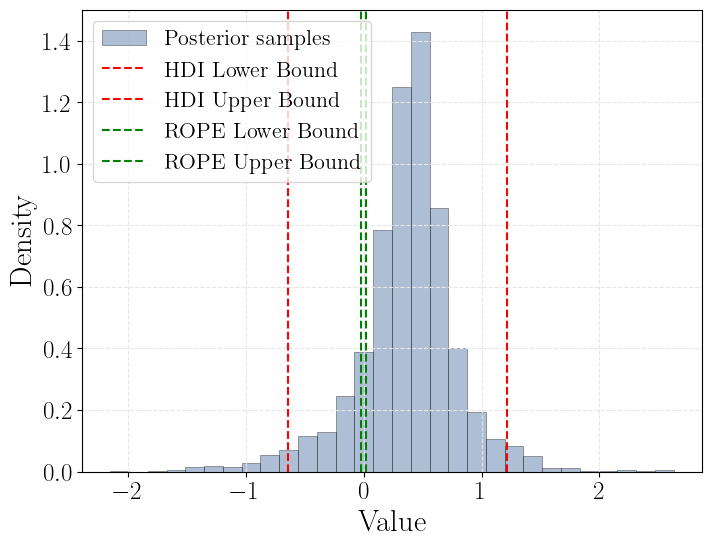}
    \caption{HDI + ROPE for local parameters left (Beta score) and global parameters right (mu score) of predictor: \texttt{fcollege\_yes}. Pooling factor of $\gamma_d=0.92$.}
    \label{fig:fcollege}
\end{figure}

\begin{figure}[ht!]
    \centering
    \includegraphics[width=0.495\textwidth]
{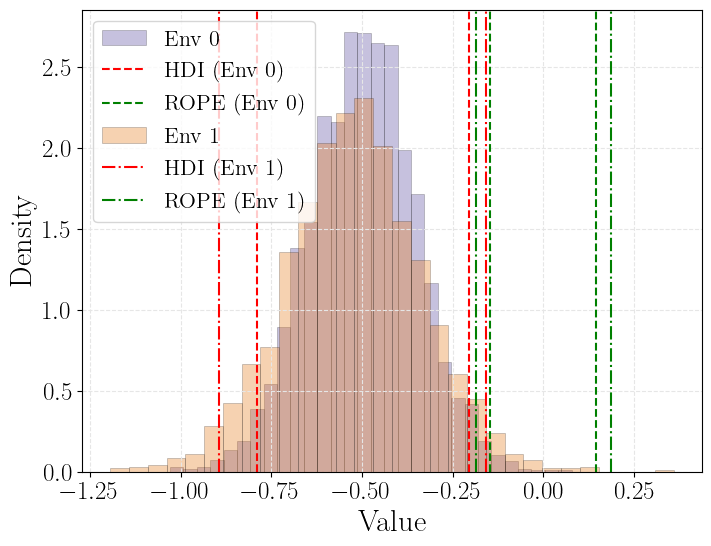}
    \includegraphics[width=0.495\textwidth]{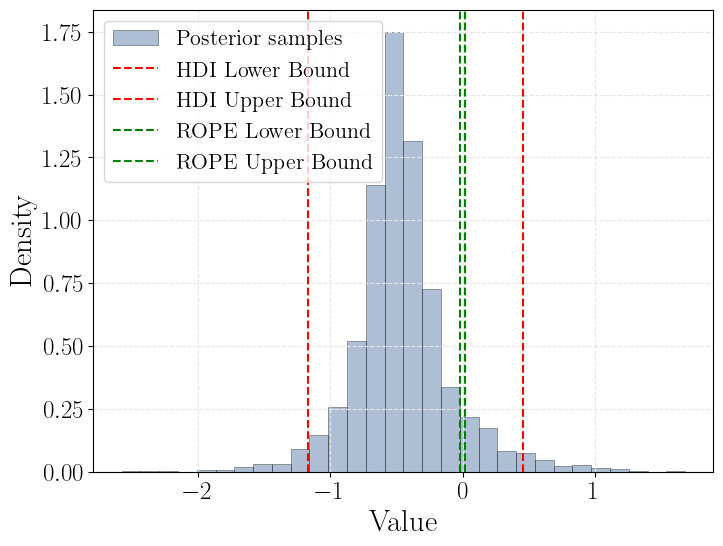}
    \caption{HDI + ROPE for local parameters left (Beta score) and global parameters right (mu score) of predictor: \texttt{income\_low}. Pooling factor of $\gamma_d=0.99$.}
    \label{fig:income}
\end{figure}

\section{Synthetic Results}\label{synthetic_results}

To rigorously validate BHIP's capability to recover the invariant parental set for each node, we conducted a significantly expanded suite of simulation experiments comparing its performance against ICP. The experimental design encompassed 18 distinct configurations, systematically varying the number of nodes $N\in\{4,5,6\}$, the number of samples per environment $S\in\{50,500,2000\}$, and the number of environments $E\in\{2,3\}$. Each setup included one observational environment and $E-1$ interventional environments, where each intervention consisted of a distinct single-node manipulation.

For each of the 18 configurations $(N,S,E)$, we generated $1000$ random DAGs. Corresponding SCMs were generated based on Linear Gaussian Additive Noise Models (LGANMs). Following the setup, we employed noise variance, drawn from a $U(0,0.3)$ distribution, and causal effects, drawn from a $U(1,5)$ distribution \citep{sempler}.

For BHIP's evaluation specific decision rules were employed. To assess parameter stability across environments, we utilized the HDI+ROPE method both on local and on global parameters. The ROPE was defined symmetrically around zero as the interval $[-0.1\hat{s},+0.1\hat{s}]$, where $\hat{s}$ represents the sample standard deviation of the target variable. A parameter stability test was considered `passed` if the posterior probability of the parameter value falling outside this ROPE was greater than $0.95$. This is equivalent to ensuring that $95\%$ of the posterior HDI mass lies outside the ROPE. Separately, a decision rule based on a pooling factor was employed, requiring this factor to exceed a threshold of $0.85$ to pass. The final acceptance of a candidate parent set in these simulations required satisfying criteria based on both the HDI+ROPE rule and the pooling factor rule. For comparison, ICP was evaluated using a standard significance level of $\alpha=0.05$.

Performance was quantified by evaluating the recovery of the true parental set for each node, using the F1 score, Recall, Precision, and Specificity metrics. The reported results for each configuration represent the average performance across the 1000 independent runs. The comparative performance metrics for BHIP and ICP across all configurations are presented in Table \ref{table:bhip_icp_comparison_no_specificity_v3} and several key observations can be made:

 BHIP generally outperforms ICP in terms of F1 score, Recall, and Precision across a majority of the configurations (14 out of 18). The advantage of BHIP is particularly pronounced in scenarios with larger sample sizes ($S=500$ or $S=2000$).
As expected, performance for both methods improves significantly with increasing sample size per environment. BHIP's relative advantage over ICP often widens as $S$ increases, suggesting BHIP benefits more effectively from larger datasets within each environmental context.
Increasing the number of environments from E=2 (one observational, one interventional) to E=3 (one observational, two distinct interventions) generally boosts performance for both methods.

In summary, the simulation results indicate that BHIP offers a substantial improvement over ICP for invariant parent set recovery under a wide range of conditions, particularly demonstrating robustness with fewer environmental interventions $E=2$ and leveraging increased sample sizes $S$ more effectively.

\begin{table}[h!]
\centering
\caption{Comparison of BHIP vs ICP Metrics}
\begin{tabular}{ l c c cccc }
\toprule
Config & \multicolumn{2}{c}{f1\_score} & \multicolumn{2}{c}{Recall} & \multicolumn{2}{c}{Precision} \\
\cmidrule(lr){2-3} \cmidrule(lr){4-5} \cmidrule(lr){6-7}
& BHIP & ICP & BHIP & ICP & BHIP & ICP \\
\midrule
N=4,S=50,E=2 & \textbf{0.2003} & 0.1538 & \textbf{0.1898} & 0.1391 & \textbf{0.2225} & 0.1857 \\
N=4,S=50,E=3 & 0.3830 & \textbf{0.4536} & 0.3693 & \textbf{0.4295} & 0.4110 & \textbf{0.5035} \\
N=4,S=500,E=2 & \textbf{0.3755} & 0.1659 & \textbf{0.3600} & 0.1525 & \textbf{0.4123} & 0.1959 \\
N=4,S=500,E=3 & \textbf{0.5621} & 0.5035 & \textbf{0.5397} & 0.4837 & \textbf{0.6088} & 0.5438 \\
N=4,S=2000,E=2 & \textbf{0.4948} & 0.1928 & \textbf{0.4790} & 0.1762 & \textbf{0.5338} & 0.2277 \\
N=4,S=2000,E=3 & \textbf{0.6411} & 0.5310 & \textbf{0.6282} & 0.5132 & \textbf{0.6695} & 0.5680 \\
N=5,S=50,E=2 & 0.0732 & \textbf{0.0830} & 0.0675 & \textbf{0.0711} & 0.0860 & \textbf{0.1123} \\
N=5,S=50,E=3 & 0.2313 & \textbf{0.2436} & 0.2176 & \textbf{0.2177} & 0.2615 & \textbf{0.3029} \\
N=5,S=500,E=2 & \textbf{0.2258} & 0.0907 & \textbf{0.2122} & 0.0786 & \textbf{0.2560} & 0.1171 \\
N=5,S=500,E=3 & \textbf{0.3690} & 0.2673 & \textbf{0.3548} & 0.2478 & \textbf{0.4053} & 0.3090 \\
N=5,S=2000,E=2 & \textbf{0.3073} & 0.1139 & \textbf{0.2948} & 0.1018 & \textbf{0.3399} & 0.1404 \\
N=5,S=2000,E=3 & \textbf{0.4574} & 0.3145 & \textbf{0.4444} & 0.2898 & \textbf{0.4967} & 0.3692 \\
N=6,S=50,E=2 & 0.0241 & \textbf{0.0302} & 0.0219 & \textbf{0.0241} & 0.0290 & \textbf{0.0470} \\
N=6,S=50,E=3 & \textbf{0.1595} & 0.1234 & \textbf{0.1512} & 0.1035 & \textbf{0.1793} & 0.1717 \\
N=6,S=500,E=2 & \textbf{0.1113} & 0.0419 & \textbf{0.1024} & 0.0346 & \textbf{0.1325} & 0.0596 \\
N=6,S=500,E=3 & \textbf{0.2529} & 0.1671 & \textbf{0.2424} & 0.1454 & \textbf{0.2866} & 0.2182 \\
N=6,S=2000,E=2 & \textbf{0.1924} & 0.0661 & \textbf{0.1777} & 0.0560 & \textbf{0.2306} & 0.0903 \\
N=6,S=2000,E=3 & \textbf{0.3383} & 0.2250 & \textbf{0.3254} & 0.2000 & \textbf{0.3812} & 0.2813 \\
\bottomrule
\end{tabular}
\label{table:bhip_icp_comparison_no_specificity_v3}
\end{table}

\section{Synthetic Benchmark Extension}
\label{appendix:precision_recall}

Following the experimental setup in Section \ref{subsec:bip_synth}, this section provides an extended evaluation of the linear-Gaussian synthetic benchmark using traditional machine learning classification metrics. Figure~\ref{fig:extended_metrics} illustrates the Precision, Recall, and F1-score for BHIP and the competing invariant prediction methods across a varying number of environments ($E \in \{2, 5, 10, 20\}$). By analyzing these metrics, we can better contextualize the exact discovery and coverage results presented in the main text.

Both BHIP and ICP maintain a perfect precision of 1.0 across all evaluated environments. This confirms the Coverage result, mirroring the strict Type I error control of ICP. Conversely, methods like EILLS and Hidden-ICP exhibit significantly lower precision in the low-environment regime ($E=2$ and $E=5$), explaining their lower Coverage despite having high recall.

While ICP is highly precise, it is also conservative as reflected in its lower recall compared to the BHIP. BHIP recall starts conservatively at $E=2$ but climbs steeply as $E$ increases, efficiently leveraging the added environmental heterogeneity to identify the true invariant set and surpassing ICP.

As seen in the rightmost panel, BHIP provides a good trade-off. BHIP's F1-score significantly outperforms traditional ICP and Hidden-ICP, and it closely tracks the state-of-the-art performance of BIP and VI-BIP. This demonstrates BHIP's capacity to maintain the rigorous Type I error control while overcoming their power deficiencies (recall) through hierarchical Bayes.

\begin{figure}[htbp]
    \centering
    \includegraphics[width=\textwidth]{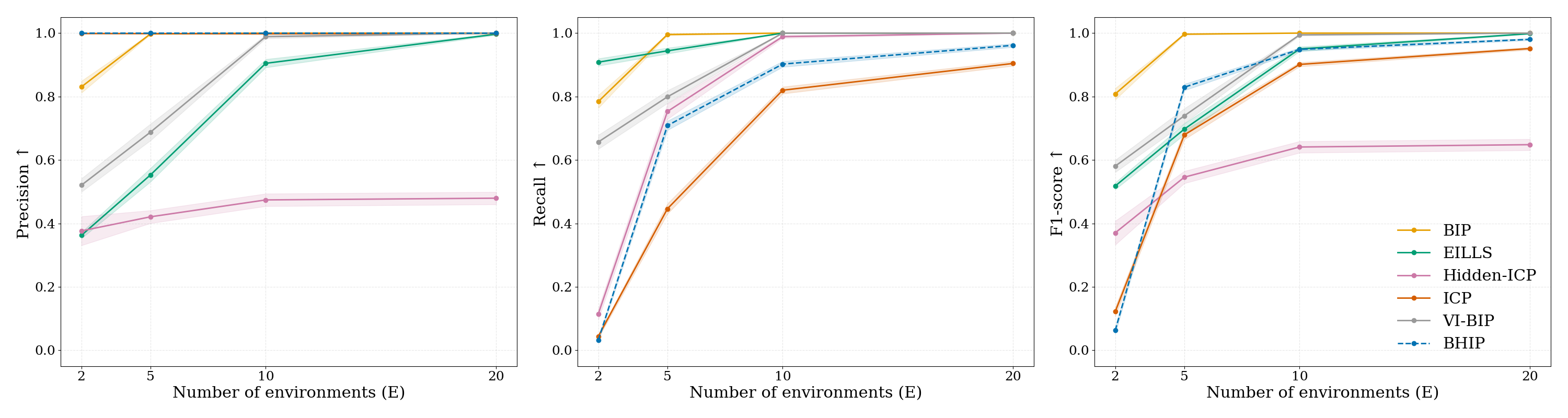}
    \caption{Extended low-dimensional synthetic benchmark. Precision, Recall, and F1-score shown as a function of the number of environments $E$.}
    \label{fig:extended_metrics}
\end{figure}

\section{TU Dataset Results}\label{tu_results}

BHIP was also applied to the TU Danish Travel Survey Dataset which surveys the transport behaviour of Danish people residing in Denmark \citep{tuDataset}. The research question can be described as: what are the causal predictors of the primary mode of transportation choice? 

After some data cleaning and feature selection, the predictors considered the respondent's age, sex, education level, purpose of trip as well as if the respondent has a bicycle. The data was split in two environments depending on the correspondent's distance from work to home. Additionally, the number of cars and the income of the respondent's household are also considered.

Both the possession of a bicycle as well as the number of cars of the respondent's household are predictors with a considerable effect on the primary transport mode of choice of the respondents. Furthermore, RespHasBicycle has a pooling factor of $0.98$ and HousehNumcars has a pooling factor of $0.99$. When considering sparse priors these results are validated and both spike-and-slab as well as horseshoe prior models select these variables as the most relevant, as seen on Table \ref{table:tu}.

\begin{table}
\centering
\caption{Sparse priors variable selection results for the TU dataset}
\begin{tabular}{ lcc } 
\toprule
Predictor & $z_d$ & $\lambda_{d}$  \\
\midrule
RespHasBicycle & 0.82 & 3.92 \\ 
HousehNumcars  & 0.62 & 2.69 \\ 
IncHouseh  & 0.05 & 0.46 \\ 
RespAgeSimple  & 0.14 & 0.88 \\ 
RespSex  & 0.15 & 0.56 \\ 
DiaryDaytype & 0.14 & 0.47\\ 
RespEdulevel  & 0.01 & 0.76 \\ 
\bottomrule
\end{tabular}
\label{table:tu}
\end{table}

\end{document}